\documentclass{article}

\PassOptionsToPackage{numbers, compress}{natbib}

\usepackage[preprint]{neurips_2023}

\usepackage[utf8]{inputenc}
\usepackage[T1]{fontenc}   
\usepackage{url}           
\usepackage{booktabs}      
\usepackage{amsfonts}      
\usepackage{nicefrac}      
\usepackage{microtype}     
\usepackage{xcolor}

\usepackage{amsmath}
\usepackage{amssymb}
\usepackage{mathtools}
\usepackage{amsthm}

\usepackage{thm-restate}
\usepackage{extarrows}

\usepackage{subfigure}

\usepackage{graphicx}
\usepackage{multirow}
\usepackage{colortbl}
\usepackage{makecell}
\usepackage[frozencache=true,cachedir=_minted-main]{minted} 
\usepackage{newfloat}
\usepackage{caption}

\definecolor{green}{RGB}{116, 199, 79}
\usepackage[colorlinks,
  citecolor=green,
  linkcolor=red,
  urlcolor=cyan
]{hyperref}

\theoremstyle{plain}
\newtheorem{theorem}{Theorem}[section]

\newtheorem{lemma}[theorem]{Lemma}
\newtheorem{corollary}[theorem]{Corollary}
\theoremstyle{definition}
\newtheorem{definition}[theorem]{Definition}
\newtheorem{assumption}[theorem]{Assumption}
\newtheorem{remark}[theorem]{Remark}

\DeclareFloatingEnvironment{codeblock}
\SetupFloatingEnvironment{codeblock}{placement=htp}
\title{A Study of Neural Collapse Phenomenon: Grassmannian Frame, Symmetry and Generalization}

\author{
  \textbf{Peifeng Gao} \textsuperscript{\rm 1},
  \textbf{Qianqian Xu} \textsuperscript{\rm 2, \thanks{Corresponding author.}},
  \textbf{Peisong Wen} \textsuperscript{\rm 1, \rm 2}, \\
  \textbf{Huiyang Shao} \textsuperscript{\rm 1, \rm 2}, 
  \textbf{Zhiyong Yang} \textsuperscript{\rm 1}, 
  \textbf{Qingming Huang} \textsuperscript{\rm 1, \rm 2, \rm 3, \rm 4,}$^{*}$ \\
  \textsuperscript{\rm 1} School of Computer Science and Tech., University of Chinese Academy of Sciences \\
  \textsuperscript{\rm 2} Key Lab of Intell. Info. Process., Inst. of Comput. Tech., CAS \\
  \textsuperscript{\rm 3} BDKM, University of Chinese Academy of Sciences \\
  \textsuperscript{\rm 4} Peng Cheng Laboratory \\
}

\begin{document}

\maketitle

\begin{abstract}
  In this paper, we extend original \textit{Neural Collapse} Phenomenon by proving 
  \textit{Generalized Neural Collapse} hypothesis. We obtain \textit{Grassmannian Frame}
  structure from the optimization and generalization of classification. 
  This structure maximally separates
  features of every two classes on a sphere and does not require a larger 
  feature dimension than the number of classes. Out of curiosity about the symmetry of \textit{Grassmannian Frame}, 
  we conduct experiments to explore if models with different \textit{Grassmannian Frames} have different performance. 
  As a result, we discover the \textit{Symmetric Generalization} phenomenon. 
  We provide a theorem to explain \textit{Symmetric Generalization} of permutation. 
  However, the question of why different directions of features can lead to such different generalization
  is still open for future investigation. 
\end{abstract}

\section{Introduction} 
Consider classification problems. Researchers with good statistical and linear algebra training may believe 
that the features learned by deep neural networks, which are flowing tensors in deep models, are very 
different and vary randomly depending on the classification situation. 
However, \citet{PNAS2020} discovered a phenomenon called \textit{Neural Collapse} (NeurCol) that challenges this expectation. 
NeurCol is based on the standard training paradigm of 
classification problems using the cross-entropy loss minimization based on the stochastic gradient descent (SGD) algorithm with deep models.
During Terminal Phase of Training (TPT), NeurCol shows that the features of the last layer and linear classifier that belong to the same class 
collapse to one point and form a symmetrical geometric structure called Simplex ETF. 
This phenomenon is beautiful due to its surprising symmetry. However, existing 
conclusions \cite{PNAS2020, ji2022an, fang2021exploring, zhu2021geometric, 
zhou2022all, NEURIPS2022_4b3cc0d1, han2022neural, tirer2022extended, zhou2022optimization, 
mixon2020neural, poggio2020explicit, graf2021dissecting, NEURIPS2022_f7f5f501}  
of NeurCol are not universal enough. Specifically, the existing 
conclusions require the feature dimension to be larger than the class number, 
as Simplex ETF only exists in this case. In another case where the feature dimension 
is smaller than the class number, the corresponding structure learned by deep models is still unclear.

\cite{lu2022neural} and \cite{liu2023generalizing} provide a preliminary answer to this question. 
\citet{lu2022neural} prove that when class number tends towards infinity, the features of different classes uniformly
distribute on the hypersphere. Further, \citet{liu2023generalizing} proposes a \textit{Generalized Neural Collapse} 
hypothesis, which states that if the class number is larger than the feature dimension, the inter-class 
features and classifiers will be maximally distant on a hypersphere, 
which they refer to as \textit{Hyperspherical Uniformity}.

\textbf{First Contribution} 
Our first contribution is the theoretical confirmation of the \textit{Generalized Neural Collapse} hypothesis.
We derive the \textit{Grassmannian Frame} from three perspectives, namely, 
optimal codes in coding theory, optimization and generalization in classification problems.
As a more general version of ETF, \textit{Grassmannian Frame} exists for any combination of class numbers and feature dimensions. 
Additionally, it exhibits the \textit{\textbf{minimized maximal correlation}} property, which is precisely the \textit{Hyperspherical Uniformity} property. 

The study conducted by \citet{NEURIPS2022_f7f5f501} is relevant to our next part of work. 
They argue that deep models can learn features with any direction, and thus, 
fixing the last layer as an ETF during training can 
lead to satisfactory performance. However, we completely disprove this argument by discovering a new phenomenon. 

\textbf{Second Contribution} 
Our second contribution is the discovery of a new phenomenon called \textbf{\textit{Symmetric Generalization}}. 
Our motivation for it stems from the two invariances of the \textit{Grassmannian Frame}, namely,
rotation invariance and permutation invariance. We observe that models that have learned the 
\textit{Grassmannian Frame} with different rotations and permutations exhibit very different 
generalization performances, even if they have achieved the best performance on the training set.

\section{Preliminary}

\subsection{Neural Collapse} 

\citet{PNAS2020} conducted extensive experiments to reveal the NeurCol phenomenon. 
This phenomenon occurs during the Terminal Phase of Training (TPT), which starts from the epoch that the training accuracy has reached $100\%$.
During TPT, training error is zero, but cross-entropy loss keeps decreasing. 
To describe this phenomenon more clearly, we introduce several necessary notations first. 
We denote the class number as $C$ and feature dimension as $d$. 
Here, we consider classifiers with the form 
$logit = \boldsymbol{M} \boldsymbol{z} = [\langle M_1, \boldsymbol{z} \rangle, \dots, \langle M_C, \boldsymbol{z} \rangle]^{T}$,
where $\boldsymbol{M} \in \mathbb{R}^{d \times C}$ is the linear classifier, and $\boldsymbol{z}$ is the feature of a smaple obtained from a deep feature extractor. 
The classification result is given by selecting the maximum score of $logit$. 
Given a balanced dataset, we denote the feature of $i$-th sample in $y$-th category as $\boldsymbol{z}_{y, i}$. 
Specifically, when the model is trained on a balanced dataset, 
its last layer would converge to the following manifestations:
\begin{itemize}
  \item[\textbf{NC1}] \textbf{Variability Collapse} All samples belonging the same class converge to the class mean: $\Vert \boldsymbol{z}_{y, i} - \bar{\boldsymbol{z}}_{y} \Vert\rightarrow 0, \forall y, \forall i$
  where $\bar{\boldsymbol{z}}_{y} = \text{Ave}_{i}\left(\boldsymbol{z}_{y, i}\right)$ denote the class-center of $y$-th class;
  \item[\textbf{NC2}] \textbf{Convergence to Self Duality} The samples and classifier belonging the same class converge to the same: $\Vert \boldsymbol{z}_{y, i} - M_{y} \Vert\rightarrow 0, \forall y, \forall i$;
  \item[\textbf{NC3}] \textbf{Convergence to Simplex ETF} The classifier weight converges to the vertices of Simplex ETF;
  \item[\textbf{NC4}] \textbf{Nearest Classification} The learned classifier behaves like the nearest classifier, $i.e. \arg \max_{y} \langle M_y, z \rangle \rightarrow \arg \min_{y} \Vert z - \bar{\boldsymbol{z}}_{y} \Vert$.
\end{itemize}
In \textbf{\textbf{NC3}}, Simplex ETF is an interesting structure. 
Note that there exists two different objects: Simplex ETF and ETF. 
ETF is rooted from Frame Theroy (refer to next subsection), 
while \citet{PNAS2020} introduced the Simplex ETF as a new definition in the context of the NeurCol phenomenon. 
They made some extensions to the original definition of ETF by introducing an 
orthogonal projection matrix and a scale factor. Here, we provide the definition of Simplex ETF: 
\begin{definition}[\textbf{Simplex Equiangular Tight Frame} \cite{PNAS2020}]
  A Simplex ETF is a collection of points in $\mathbb{R}^{C}$ specified by the columns of 
  \begin{equation*}
  \begin{aligned}
    \boldsymbol{M}^{\star} = \alpha R 
    \sqrt{\frac{C}{C-1}}
    \left( I - \frac{1}{C} \mathbb{I} \mathbb{I}^{T} \right)
  \end{aligned}
  \end{equation*}
  where $I \in \mathbb{R}^{C \times C}$ is the identity matrix, $\mathbb{I} \in \mathbb{R}^{C}$ is the all-one vector,
  $R \in \mathbb{R}^{d \times C} (d \ge C)$ is an orthogonal projection matrix, $\alpha \in \mathbb{R}$ is a scale factor.
\end{definition}
Simplex ETF is a structure with high symmetry, as every pair of vectors has equal angles (\textit{Equiangular} property). 
However, it has a limitation: it only exists when feature dimension $d$ is larger than class number $C$. 
Recently, a work \cite{liu2023generalizing} removed this restriction by proposing 
\textit{Generalized Neural Collapse} hypothesis. In their hypothesis, 
\textit{Hyperspherical Uniformity} is introduced to generalize \textit{Equiangular} property. 
\textit{Hyperspherical Uniformity} means features of every class is distributed uniformly on a hypersphere with maximal distance. 

\subsection{Frame Theory}
We have discovered that \textit{Grassmannian Frame}, a mathematical object from Frame Theory, is a suitable
candidate for meeting \textit{Hyperspherical Uniformity}. 
Frame Theory is a fundamental research area in mathematics \cite{casazza2012finite} that provides a framework for the analysis of signal transmission and reconstruction. 
In communication field, certain frame properties have been shown to be optimal 
configurations in various transmission problems.
For instance, \textit{Uniform} and \textit{Tight} frame are optimal codes in erasure problem \cite{casazza2003equal} 
and non-orthogonal communication schemes \cite{ambrus2021uniform}. 
\textit{Grassmannian Frame} \cite{holmes2004optimal, strohmer2003grassmannian}, 
as a more specialized example, not only satisfies the \textit{Uniform} and \textit{Tight} 
properties but also meets the \textbf{\textit{minimized maximal correlation}} property. 
This property makes us confident that \textit{Grassmannian Frame} satisfies \textit{Hyperspherical Uniformity}. 

Here, we provide a series of definitions in Frame Theory\footnote{
  Frame Theory has more general definitions in Hilbert space.
  Definitions we provided here are actually based on Euclidean space. 
  See \cite{strohmer2003grassmannian} for general definitions.

}:
\begin{definition}[\textbf{Frame}]
  In Euclidean space $\mathbb{R}^{d}$, a frame is a sequence of bounded vectors $\{\zeta_i\}_{i=1}^{C}$.
\end{definition}
\begin{definition}[\textbf{Uniform Property and Unit Norm}]
  Given a frame $\{\zeta_i\}_{i=1}^{C}$ in $\mathbb{R}^{d}$, 
  it is a \textit{uniform} frame if the norm of every vector is equal.
  Further, it is a \textit{unit norm} frame if the norm of every vector in it is equal to 1.
\end{definition}
\begin{definition}[\textbf{Tight Property}]
  Given a frame $\{\zeta_i\}_{i=1}^{C}$ in $\mathbb{R}^{d}$, 
  it is a \textit{tight} frame if the rank of its analysis matrix $\left[\zeta_1, \dots, \zeta_C\right]$ is $d$.
\end{definition}
\begin{definition}[\textbf{Maximal Frame Correlation}]
  Given a unit norm frame $\{\zeta_i\}_{i=1}^{C}$, the \textit{maximal correlation} is defined
  as $\mathcal{M}(\{\zeta_i\}_{i=1}^{C}) =  \max_{i,j,i \neq j} \{ |\langle \zeta_i, \zeta_j \rangle| \}$.
\end{definition}
We can now define the \textit{Grassmannian frame}:
\begin{definition}[\textbf{Grassmannian Frame}]
  A frame $\{\zeta_i\}_{i=1}^{C}$ in $\mathbb{R}^{d}$ is \textit{Grassmannian frame} 
  if it is the solution of $\min \{\mathcal{M}(\{\zeta_i\}_{i=1}^{C})\}$,
  where the minimum is taken over all unit norm frames in $\mathbb{R}^{d}$.
\end{definition}
We also introduce \textit{Equiangular Tight Frame} (ETF). 
\begin{definition}[\textbf{Equiangular Property}]
  Given a unit norm frame $\{\zeta_i\}_{i=1}^{C}$, 
  it is \textit{equiangular} frame if 
  $| \langle \zeta_i, \zeta_j \rangle | = c, \forall i, j \ \text{with}\ i \neq j $
  for some constant $c \ge 0$.
\end{definition}
\begin{definition}[\textbf{Equiangular Tight Frame}]
  ETF is a \textit{Equiangular} and \textit{Tight} frame.
\end{definition}
The following theorem relates ETF and \textit{Grassmannian Frame}:
\begin{theorem}[\textbf{Welch Bound} \cite{welch1974lower}]\label{capacity}
  Given any unit norm frame $\{\zeta_i\}_{i=1}^{C}$ in $\mathbb{R}^{d}$, then we have 
  \begin{equation*}
  \begin{aligned}
    \mathcal{M}(\{\zeta_i\}_{i=1}^{C}) \ge \sqrt{\frac{C-d}{d(C-1)}} \ \ \text{only if} \ \ C \le \frac{d(d+1)}{2},
  \end{aligned}
  \end{equation*}
  Further, 
  $\mathcal{M}(\{\zeta_i\}_{i=1}^{C})$ reaches the right hand side if and only if 
  $\{\zeta_i\}_{i=1}^{C}$ is a \textit{Equiangular Tight Frame}.
\end{theorem}

This theorem tells ETF is the special case of \textit{Grassmannian Frame} and 
how a \textit{Grassmannian frame} $\{\zeta_i\}_{i=1}^{C}$ can be \textit{Equiangular}: 
if and only if it can achieve the Welch Bound. 
Intuitively, if $d$ is large enough, 
the correlation between every two vectors in $\{\zeta_i\}_{i=1}^{C}$ can be minimized equally
to achieve \textit{Equiangular} property. 
Otherwise, \textit{Equiangular} property can not be guaranteed. 

\textbf{Motivation}
Our motivation for considering \textit{Grassmannian Frame} as a potential structure for the 
\textit{Generalized Neural Collapse} hypothesis is that it satisfies two important properties: 
it has \textit{\textbf{minimized maximal correlation}} ($i.e.$, \textit{Hyperspherical Uniformity}) and exists for any vector number $C$ and 
dimension $d$. We will provide theoretical supports for this argument in the next section.

\section{Main Results}
As discussed in the previous section, \textbf{\textit{minimized maximal correlation}} is a key characteristic of \textit{Grassmannian frame}. 
Therefore, in this section, 
all of our findings are based on this insight: 
\begin{equation*}
\begin{aligned}
  \min_{\Vert M_y \Vert = 1} \max_{y \neq y'} \langle M_y, M_{y'} \rangle
\end{aligned}
\end{equation*}
All proofs can be found in the Appendix.

\subsection{Warmup: Optimal Code Perspective}
In communication field, \textit{Grassmannian frame} is not only the optimal $2$-erasure code \cite{holmes2004optimal}, 
but also the optimal code in Gaussian Channel \cite{PNAS2020, shannon1959probability}:
\begin{restatable}[]{theorem}{communicationperspective}\label{communicationperspective}
  Consider the communication problem: a number $c$ ($c \in [C]$) is encoded as a vector $M_c$ in $\mathbb{R}^{d}$ that we call \textit{code}, 
  and then is transmitted over a noisy channel. A receiver need to revocer $c$ from the noisy signal $\boldsymbol{h} = M_c + \boldsymbol{g}$, where 
  $\boldsymbol{g}$ is the additive noisy. 
  Then if $\boldsymbol{g} \sim \mathcal{N}(0, \sigma^2 I)$, \textit{Grassmannian Frame} is the optimal code enjoying the minimal error probability.
\end{restatable} 
This theorem is essentially adopted from the Corollary.4 of \cite{PNAS2020}, 
which is the first study to identify NeurCol phenomenon. 
However, they only validated this result with Simplex ETF and did not further investigate this structure. 

\subsection{Optimization Perspective}
Then we consider the classification from the optimization perspective. 
We start from the cross-entropy minimization problem. 

\textbf{Notations}
Denote feature dimension of the last layer as $d$ and class number as $C$. 
The linear classifier is $\boldsymbol{M} = [M_1, \cdots, M_{C}] \in \mathbb{R}^{d \times C}$. 
Given a balanced dataset $\boldsymbol{Z} = \{ \{ \boldsymbol{z}_{y,i} \in \mathbb{R}^{d} \}_{i=1}^{N/C} \}_{y=1}^{C}$ 
where every class has $N/C$ samples, 
we use $\boldsymbol{z}_{y,i}$ to represent the feature of $i$-th sample in $y$-th class.  

\textbf{Optimization Objective}
Since modern deep models are highly overparameterized, we assume the models have infinite capacity and can fit any dataset. 
Therefore, we directly optimize sample features to simplify analysis \cite{NEURIPS2022_f7f5f501, zhu2021geometric}: 
\begin{equation}\label{opt0}
  \begin{aligned}
    \min_{\boldsymbol{Z}, \boldsymbol{M}} \text{CELoss}(\boldsymbol{M}, \boldsymbol{Z}) := - 
    \sum_{y=1}^{C} \sum_{i=1}^{N/C} \log \frac{\exp \big(\langle M_{y}, \boldsymbol{z}_{y, i} \rangle \big)}{\sum_{y'} \exp \big(\langle M_{y'}, \boldsymbol{z}_{y, i} \rangle\big)}
\end{aligned}
\end{equation}
As the Proposition.2 of \cite{fang2021exploring} highlighted, NeurCol occurs only if features and classifiers are $\ell_2$ norm bounded. 
Therefore, following previous work \cite{lu2022neural, fang2021exploring, yaras2022neural}, 
we introduce $\ell_2$ norm constraint into (\ref{opt0}):
\begin{equation}\label{opt1}
\begin{aligned}
  s.t. \ \ \Vert M_y \Vert \le \rho, \forall y \in [C] \ \ \text{and} \  \ 
  \Vert \boldsymbol{z}_{y, i} \Vert \le \rho, \forall y \in [C], i \in [N/C]
\end{aligned},
\end{equation}
where the norms of features and linear classifiers are bounded by $\rho$. 
Then to perform optimization directly, we turn to the following unconstrained feature models
\cite{zhu2021geometric, zhou2022all}: 
\begin{equation}\label{opt2}
\begin{aligned}
  \min_{\boldsymbol{M}, \boldsymbol{Z}} \mathcal{L}(\boldsymbol{M}, \boldsymbol{Z}) := 
  \sum_{y=1}^{C} \sum_{i=1}^{N/C} \left( -\log 
    \frac{\exp \big(\langle M_{y}, \boldsymbol{z}_{y, i} \rangle \big)}{\sum_{y'} \exp \big(\langle M_{y'}, \boldsymbol{z}_{y, i} \rangle\big)} + 
    \frac{\omega \Vert \boldsymbol{z}_{y, i} \Vert^2}{2}
  \right) 
  + 
  \sum_{y=1}^{C} \frac{\lambda \Vert M_y \Vert^2}{2}
\end{aligned}
\end{equation}
Compared with (\ref{opt0}), $\Vert M_y \Vert^2$ and $\Vert \boldsymbol{z}_{y, i} \Vert^2$ are added in (\ref{opt2}). 
In deep learning, the newly added terms can be seen as the weight decay, 
while $\lambda$ and $\omega$ are factors of weight decay. 
This can limit the norm of the sample features $\boldsymbol{Z}$ and linear classifier $\boldsymbol{M}$, 
that is, given any $\lambda, \omega \in (0, \infty)$, there exists a $\rho(\lambda, \omega)$ 
such that $\rho(\lambda, \omega) \ge \Vert M_y \Vert$ and 
$\rho(\lambda, \omega) \ge \Vert \boldsymbol{z}_{y, i} \Vert$ when (\ref{opt2}) converges.


\textbf{Gradient Descent} 
\cite{ji2022an} analyses Gradient Flow of (\ref{opt0}), while 
different from them, we consider the Gradient Descent of (\ref{opt2}): 
\begin{equation}\label{GDupdate}
\begin{aligned}
  \boldsymbol{Z}^{(t+1)} \leftarrow \boldsymbol{Z}^{(t)} - \alpha \nabla_{\boldsymbol{Z}^{(t)}} \mathcal{L}(\boldsymbol{M}^{(t)}, \boldsymbol{Z}^{(t)}) 
  \ \ \text{and} \ \ 
  \boldsymbol{M}^{(t+1)} \leftarrow \boldsymbol{M}^{(t)} - \beta \nabla_{\boldsymbol{M}^{(t)}} \mathcal{L}(\boldsymbol{M}^{(t)}, \boldsymbol{Z}^{(t)})
\end{aligned}
\end{equation}
where we use upper script $(t)$ to denote optimized variables in the $t$-th iterations, 
$\alpha$ and $\beta$ are learning rates. 
Note that when $\lambda, \omega \rightarrow 0$, 
the optimization problem (\ref{opt2}) is equivalent to (\ref{opt0})
since norm constraint condition (\ref{opt1}) vanishes, $i.e. \rho \rightarrow \infty$.

\begin{restatable}[\textbf{Generalized Neural Collapse}]{theorem}{CEgrass}\label{CE_grass}
  Consider the convergce of Gradient Descent on the model (\ref{opt2}), 
  if parameters are properly selected (refer to Assumption.\ref{assumption_} in Appendix for more details),
  we have the following conclusions:

  \textbf{\textit{(NC1)}} $\Vert \boldsymbol{z}_{y,i} - \boldsymbol{z}_{y,j} \Vert \rightarrow 0, \forall y \in [C], \forall i, j \in [N/C]$;

  \textbf{\textit{(NC2)}} $\Vert \boldsymbol{z}_{y,i} - M_{y} \Vert \rightarrow 0, \forall y \in [C], \forall i \in [N/C]$;

  \textbf{\textit{(NC3)}} if $\rho \rightarrow \infty$, $\boldsymbol{M}$ converges to the solution of $\min_{\forall y, \Vert M_y \Vert = \rho} \max_{y \neq y'} \langle M_y, M_{y'}\rangle$;

  \textbf{\textit{(NC4)}} $\arg \max_{y} \langle M_{y}, \boldsymbol{z}\rangle \rightarrow \arg \max_{y} \Vert \boldsymbol{z} - \bar{\boldsymbol{z}}_{y} \Vert
  \ \forall \boldsymbol{z} \in \boldsymbol{Z}$, where $\bar{\boldsymbol{z}}_{y} = \frac{C}{N} \sum_{i=1}^{N/C} \boldsymbol{z}_{y, i}$.
\end{restatable}

\begin{remark}
  Our findings confirm \textit{Generalized Neural Collapse} hypothesis of \citet{liu2023generalizing}. 
  \textbf{\textit{(NC1)}}, \textbf{\textit{(NC2)}}, \textbf{\textit{(NC4)}} are consistent with previous discoveries \cite{PNAS2020}. 
  Additionally, we have shown that \textbf{\textit{(NC3)}} extends NeurCol's ETF and 
  leads to \textbf{\textit{minimized maximal correlation}}, or \textit{Grassmannian Frame}. 
\end{remark}

\begin{remark}
  Our findings reveal two objectives of NeurCol that \citet{liu2023generalizing} highlighted: 
  minimal intra-class variability and maximal inter-class separability. 
  Our conclusions on \textbf{\textit{(NC1)}}, \textbf{\textit{(NC2)}}, \textbf{\textit{(NC4)}} 
  support the former objective, 
  while the \textit{Grassmannian Frame} resulting from \textbf{\textit{(NC3)}} 
  naturally coincides with the solutions 
  of problems such as the Spherical Code \cite{conway2013sphere}, Thomson problem \cite{Thomson}, and Packing Lines 
  in Grassmannian Spaces \cite{conway1996packing}, which supports the latter objective.
\end{remark}

\begin{remark}
  \textbf{\textit{(NC1)}} and \textbf{\textit{(NC2)}} implies that 
  classifiers $\{M_{y}\}_{y=1}^{C}$ forms an \textit{alternate dual frame} \cite{ambrus2021uniform} of 
  features $\{\bar{\boldsymbol{z}}_{y}\}_{y=1}^{C}$. 
  In Frame Theroy, \textit{alternate dual frame} has been proved 
  to be an optimal dual with respect to erasures for decoding \cite{LENG20111464, LOPEZ2010471}. 
\end{remark}

We conduct a simulation experiment to visualize the convergence of \textit{Generalized Neural Collapse}
in a $2$-dimensional feature space with $4$ classes. 
A GIF animation can be found \href{https://www.mediafire.com/view/1tjzxx1b8t70xss/Fig2.gif/file}{HERE}.
Please refer to Appendix.\ref{GNCfigure} for a detailed description and visual 
representation of this experiment.

\subsection{Generalization Perspective}\label{GeneralizationPerspective}
Next, we consider the generalization of classification problems. 
While correlation measures 
the similarity between two vectors in Frame Theory, in classification problems, margin is a 
similar concept but with an opposite degree. Therefore, our analysis of generalization focuses on margin. 

\textbf{Notations} 
Consider a $C$ classes classification problem. 
Suppose sample space is $\mathcal{X} \times \mathcal{Y}$, where 
$\mathcal{X}$ is data space and $\mathcal{Y} = \{1, \dots, C\}$ is label space. 
We assume class distribution is $\mathcal{P}_{\mathcal{Y}} = \left[ p(1), \dots, p(C) \right]$, 
where $p(c)$ denote the proportion of class $c$. 
Let the training set $S = \{(\boldsymbol{x}_i, y_i)\}_{i=1}^{N}$ be drawn i.i.d from 
probability $\mathcal{P}_{\mathcal{X} \times \mathcal{Y}}$. 
For $y$-class samples in $S$, we denote $S_y = \{\boldsymbol{x} |(\boldsymbol{x}, y) \in S\}$ and $|S_y| = N_y$. 
The form of classifiers  is $logit = \boldsymbol{M}^{T} f(\boldsymbol{x};\boldsymbol{w}) = [\langle M_1, f(\boldsymbol{x};\boldsymbol{w}) \rangle, \dots, \langle M_C, f(\boldsymbol{x};\boldsymbol{w}) \rangle]$, 
where $\boldsymbol{M} \in \mathbb{R}^{d \times C}$ is the last linear layer, and $f(\cdot; \boldsymbol{w}) \in \mathbb{R}^{d}$ is the feature extractor parameterized by $\boldsymbol{w}$. 
We use a tuple $(\boldsymbol{M}, f(\cdot; \boldsymbol{w}))$ to denote the classifier.

First, give the definition of margin:
\begin{definition}[Linear Separability]
  Given the dataset $S$ and a classifier $(\boldsymbol{M}, f(\cdot; \boldsymbol{w}))$, 
  if the classifier can achieve $100\%$ accuracy on dataset, 
  it must have $\boldsymbol{M}$ can linearly separate the feature of $S$:
  for any two classes $i, j (i \neq j)$, there exists a $\gamma_{i,j} > 0$ such that
  \begin{equation*}
  \begin{aligned}
    & (M_i - M_j)^{T} f(\boldsymbol{x}; \boldsymbol{w}) \ge \gamma_{i,j}, &\forall (\boldsymbol{x}, i) \in S, \\
    & (M_i - M_j)^{T} f(\boldsymbol{x}; \boldsymbol{w}) \le -\gamma_{i,j}, & \forall (\boldsymbol{x}, j) \in S,
  \end{aligned}
  \end{equation*}
  In this case, we say the classifier $(\boldsymbol{M}, f(\cdot; \boldsymbol{w}))$ can linearly separate the dataset $S$ by margin $\{\gamma_{i,j}\}_{i \neq j}$. 
\end{definition}
The following lemma establishes the relationship between margin and correlation in NeurCol:
\begin{restatable}[]{lemma}{correlationmarginn}\label{correlationmargin}
  Given the dataset $S$ and a classifier $(\boldsymbol{M}, f(\cdot; \boldsymbol{w}))$, 
  if the classifier can linearly separate the dataset by margin $\{\gamma_{i,j}\}_{i \neq j}$ 
  and achieves NeurCol phenomenon on it, we have the following conclusion:
  \begin{equation*}
  \begin{aligned}
    \forall i, j \in [C] (i \neq j), \gamma_{i, j} + \langle M_i, M_j \rangle = \rho^{2}
  \end{aligned}
  \end{equation*}
\end{restatable}

By substituting the conclusions of \textbf{NC1-3} into the definition of margin, we can prove this lemma straightforwardly. 
It says given the maximal norm $\rho$, margin $\gamma_{i, j}$ and correlation $\langle M_i, M_j\rangle$ is a pair of opposite quantities. 
Then we propose the Multiclass Margin Bound. 

\begin{restatable}[\textbf{Multiclass Margin Bound}]{theorem}{maintheoremone}\label{main_theorem_0}
  Consider a dataset $S$ with $C$ classes. 
  For any classifier $(\boldsymbol{M}, f(\cdot; \boldsymbol{w}))$, 
  we denote its margin between $i$ and $j$ classes as $(M_{i} - M_{j})^{T} f(\cdot; \boldsymbol{w})$.
  And suppose the function space of the margin is $\mathcal{F} = \{ (M_{i} - M_{j})^{T} f(\cdot; \boldsymbol{w}) | \forall i \neq j, \forall \boldsymbol{M}, \boldsymbol{w}\}$, 
  whose uppder bound is 
  \begin{equation*}
  \begin{aligned}
    \sup_{i \neq j} \sup_{\boldsymbol{M}, \boldsymbol{w}} \sup_{x \in \mathcal{M}_i} \left| (M_i - M_j)^{T} f(\boldsymbol{x};\boldsymbol{w}) \right| \le K.
  \end{aligned}
  \end{equation*}
  Then, for any classifier $(\boldsymbol{M}, f(\cdot; \boldsymbol{w}))$ and margins $\{\gamma_{i,j}\}_{i\neq j} (\gamma_{i,j} > 0)$, the following inequality holds with probability at least $1 - \delta$
  \begin{equation*}
    \begin{aligned}
      \mathbb{P}_{x,y}\Big(\max_{c} [M f(\boldsymbol{x};\boldsymbol{w})]_{c} \neq y\Big) 
      \lesssim  & 
      \sum_{i=1}^{C} p(i) \sum_{j \neq i} \frac{\mathfrak{R}_{N_i}(\mathcal{F})}{\gamma_{i,j}} + 
      \sum_{i=1}^{C} p(i) \sum_{j \neq i} \sqrt{\frac{\log(\log_{2} \frac{4K}{\gamma_{i,j}})}{N_i}} 
      + L_{0,1} 
    \end{aligned}
    \end{equation*}
    where $\lesssim$ means we omit probability related terms, and $L_{0,1}$ denotes the empirical risk term:
    \begin{equation*}
    \begin{aligned}
      L_{0,1} = 
      \sum_{i=1}^{C} p(i) \sum_{j \neq i}
      \sum_{x \in S_i} \frac{\mathbb{I}((M_i - M_j)^{T}f(x) \le \gamma_{i,j})}{N_i}
    \end{aligned}
    \end{equation*} 
    $\mathfrak{R}_{N_i}(\mathcal{F})$ is the Rademacher complexity \cite{marginBound, Rademacher} of function space $\mathcal{F}$. 
    Refer to Appendix.\ref{ProofofTheorem45} for full version of this theorem. 
\end{restatable}
Recall NeurCol occurs when class distribution is uniform. We consider this case. 
\begin{corollary}
  In Theorem.\ref{main_theorem_0}, we assume the the class distribution and train set are both uniform, $i.e.$
  $p(i) = \frac{1}{C}$ and $N_i = \frac{N}{C} \ \forall i \in [C]$. 
  In this case, the generalization bound in Theorem.\ref{main_theorem_0} becomes
  \begin{equation*}
  \begin{aligned}
    \mathbb{P}_{x,y}\Big(\max_{c} [M f(\boldsymbol{x};\boldsymbol{w})]_{c} \neq y \Big) 
    \lesssim 
    \frac{\mathfrak{R}_{N/C}(\mathcal{F})}{C}\sum_{i=1}^{C}  \sum_{j \neq i} \frac{1}{\gamma_{i,j}} + 
    \frac{1}{\sqrt{NC}} \sum_{i=1}^{C} \sum_{j \neq i} \sqrt{\log(\log_{2} \frac{4K}{\gamma_{i,j}})}
    + L_{0,1} 
  \end{aligned}
  \end{equation*}
  Observing both above terms are the form $\sum_{i \neq j} \frac{1}{\gamma_{i, j}}$, 
  we have 
  \begin{equation*}
  \begin{aligned}
    \sum_{c=1}^{C} \sum_{j \neq i} \frac{1}{\gamma_{i,j}} \le C(C-1) \max_{i \neq j} \frac{1}{\gamma_{i,j}}
    \ \ \Leftrightarrow \ \ 
    \frac{1}{C(C-1)} \min_{\{\gamma_{i,j}\}_{i\neq j}} \sum_{c=1}^{C} \sum_{j \neq i} \frac{1}{\gamma_{i,j}} \le 
    \min_{\{\gamma_{i,j}\}_{i\neq j}} \max_{i \neq j} \frac{1}{\gamma_{i,j}}
  \end{aligned}
  \end{equation*}
\end{corollary}
\begin{remark}
  Once again, we observe the characteristic of \textit{\textbf{minimized maximal correlation}}. 
  However, this time it is in the form of margin and is obtained by minimizing the margin generalization error.
\end{remark}
\begin{remark}
  The Multiclass Margin Bound can provide an explanation for the steady improvement in test accuracy and 
  adversarial robustness during TPT (as shown in Figure 8 and Table 1 of \cite{PNAS2020}). 
  At the beginning of TPT, the accuracy over the training set reaches $100\%$ and $L_{0,1} = 0$, 
  indicating that generalization performance can no longer improve by reducing $L_{0,1}$. 
  However, if we continue training at this point, the margin $\gamma_{i, j}$
  would still increase. 
  Therefore, better robustness can be achieved by increasing the margin. Furthermore, 
  two terms in our bound related to margin continue to decrease, leading to better generalization performance. 
\end{remark}

\textbf{Minority Collapse} \citet{fang2021exploring} have identified a related phenomenon called \textit{Minority Collapse}, 
which is an imbalanced version of NeurCol. Specifically, they observed that when 
the training set is extremely imbalanced, the features of the minority class tend 
to converge to be parallel. Our Multiclass Margin Bound can offer explanations for the generalization of this phenomenon. 

\begin{corollary}
  Consider imbalanced classification. Given dataset $S$ with $C$ classes, the first $C_{1}$ classes 
  $\mathcal{C}_1 = \{1, \dots, C_1\}$ each 
  contain $N_1$ data, and the remaining $C - C_1$ minority classes 
  $\mathcal{C}_2 = \{C_1+1, \dots, C\}$ each contain $N_2$ data. 
  We denote the imbalanced ratio $N_1 / N_2$ as $R$. Assume class distribution $p(i)$ is the same as dataset $S$'s. 

  Then terms related to the margins between minority classes in Multiclass Margin Bound becomes:
  \begin{equation*}
  \begin{aligned}
    \sum_{i \in \mathcal{C}_2} \sum_{j \in \mathcal{C}_2 \backslash \{i\}} 
    \frac{1}{C_{1} R + C - C_1}
    \left(
      \frac{\mathfrak{R}_{N_2}(\mathcal{F})}{\gamma_{i, j}} 
      +
      \sqrt{\frac{\log\left(\log_{2}\frac{4K}{\gamma_{i, j}}\right)}{N_2}}
    \right)
  \end{aligned}
  \end{equation*}
\end{corollary}
\begin{remark}
  In these terms, $R$ and $\gamma_{i, j} (i, j \in \mathcal{C}_2)$ are inversely related in terms of their values. 
  This implies that as $R \rightarrow \infty$, if the generalization bound remains constant, then $\gamma_{i, j}$ must approach $0$. 
  Recall $\left( M_i - M_j \right)^{T} f(\boldsymbol{x};\boldsymbol{w}) \le \gamma_{i, j} \ \forall \boldsymbol{x} \in S_{i}$, which means that 
  $\Vert M_i - M_j \Vert \rightarrow 0$. 
\end{remark}

\section{Further Exploration}
In this section, we uncover a new phenomenon in NeurCol, which we refer to as \textbf{\textit{Symmetric Generalization}}. 
\textbf{\textit{Symmetric Generalization}} is linked to two transformation groups over 
\textit{Grassmannian Frame}: namely Permutation and Rotation transformation. 
Briefly, \textit{Grassmannian Frame} in these two transformations can result in different generalization performances. 
Upon observing this intriguing phenomenon in our experiments, we provide a theoretical 
result to explain this phenomenon partially. 

\subsection{Motivation}\label{motivation}
First, we introduce two kinds of equivalence of frame \cite{holmes2004optimal, bodmann2005frames}:
\begin{definition}[\textbf{Equivalent Frame}]
  Given two frames $\{\zeta_{i}\}_{i=1}^{C}, \{\chi_{i}\}_{i=1}^{C}$ in $\mathbb{R}^{d}$, 
  they are: 
  \begin{itemize}
    \item \textit{Type I Equivalent} if there exists a orthogonal matrix $R \in \mathbb{R}^{d}$ such that $[\zeta_i]_{i=1}^{C} = R [\chi_i]_{i=1}^{C}$.
    \item \textit{Type II Equivalent} if there exists a permutation matrix $P \in \mathbb{R}^{C}$ such that $[\zeta_i]_{i=1}^{C} = [\chi_i]_{i=1}^{C} P$.
  \end{itemize}
\end{definition}

The \textit{Grassmannian Frame} is a geometrically symmetrical structure, with its symmetry stemming 
from the invariance of two transformations: rotation and permutation. Specifically, after a 
rotation or permutation (or their combination), the frame still satisfies the \textit{minimized maximal correlation} 
property, as only the frame's direction and order change. We are curious about how 
these equivalences affect the performance of models.
In machine learning, we typically consider two aspects of a model: optimization and 
generalization. Given that the training set and classification model (backbone and capacity) 
are the same, we argue that models with equivalent \textit{Grassmannian Frames} will exhibit the same 
optimization performance. However, there is no reason to believe that these models will also 
have the same generalization performance. As such, we propose the following question: 

\textbf{Is generalization of models invariable to symmetric transformations of \textit{Grassmannian Frame}?}

To explore this question, we conduct a series of experiments. Our experimental results lead to an 
interesting conclusion: 

\textbf{
  Optimization performance of models is not affected by Rotation and Permutation 
  transformation of \textit{Grassmannian Frame}, 
  but generalization performance is. 
}

This newly discovered phenomenon, which we call \textbf{\textit{Symmetric Generalization}}, 
contradicts a recent argument made in \cite{NEURIPS2022_f7f5f501}. The authors of that 
work claimed that since modern neural networks are often overparameterized and 
can learn features with any direction, fixing linear classifiers as a Simplex ETF is sufficient, 
and there is no need to learn it. Our findings challenge this viewpoint. 

\subsection{Experiments}\label{Experiments_main}

To investigate the impact of Rotation and Permutation transformations of Grassmannian Frame 
on the generalization performance of deep neural networks, we conducted a series of experiments. 

\textbf{How to Reveal the Anwser} 
We generate $10$ \textit{Grassmannian Frame} with different Rotation and Permutation. 
Then, we train the same network architecture 10 times. In each time, the linear classifier are loaded from a pre-generated 
equivalent \textit{Grassmannian Frame} and fixed during training. 
To ensure a completely same optimization process (mini-batch, augmentation, and parameter initialization), 
we use the same random seed for each training. 
Once the NeurCol phenomenon occurrs, we know that the $10$ models have learned 
different \textit{Grassmannian Frames}. Finally, we compare their generalization performances by 
evaluating the cross-entropy loss and accuracy on a test set. 

\textbf{Generation of Equivalent \textit{Grassmannian Frame}} 
Our Theorem.\ref{CE_grass} naturally offers a numerical method for generating \textit{Grassmannian Frame}. 
Given class number and feature dimension, we use Gradient Descent (\ref{GDupdate}) on the 
unconstrained feature models (\ref{opt2}) to generate \textit{Grassmannian Frame}. 
Then given a \textit{Grassmannian Frame} $\{M_{y}\}_{y=1}^{C}$ in $\mathbb{R}^{d}$, 
if it is denoted as $\boldsymbol{M} = [M_1, \cdots, M_{C}] \in \mathbb{R}^{d \times C}$,
we can use $R \boldsymbol{M} P$ to denote its equivalent frame. 
where $P \in Permutation(C)$ and $R \in SO(d)$: 
\begin{equation*}
\begin{aligned}
  & Permutation(C) = \left\{ P \bigg| \forall i \in [C], \sum_{j=1}^{C} P_{j, i} = \sum_{j=1}^{C} P_{i, j} = 1 , \forall i,j \in [C], P_{i,j} \in \{0, 1\}, P \in \mathbb{R}^{C} \right\} \\
  & SO(d) = \left\{ R \big| R^{T}R = R R^{T} = I_{d}, |R| = 1, R \in \mathbb{R}^{d} \right\}
\end{aligned}
\end{equation*}
Note that $Permutation(C)$ and $SO(d)$ act on vector orders and directions of $\{M_{y}\}_{y=1}^{C}$ respectively. 
Refer to Appendix.\ref{code} for code implementation. 

\textbf{Models with Different Features} 
\citet{NEURIPS2022_f7f5f501} point out in their Theorem.1: if the linear classifier
is fixed as Simplex ETF, then the final features learned by the model 
would converge to be Simplex ETF with the same direction to classifier. 
Following their work, 
to let models to learn equivalent \textit{Grassmannian Frame}, 
we initialize the linear classifier as equivalent \textit{Grassmannian Frame}
and do not perfrom optimization on it during training. In this way, when NeurCol occurs, 
models have learned equivalent \textit{Grassmannian Frame}. 

\textbf{Network Architecture and Dataset} 
Our experiments involve two image classification datasets: CIFAR10/100 \cite{Krizhevsky2009LearningML}. 
And for every dataset, we use three different convolutional neural networks to verify our finding, 
including ResNet \cite{he2016deep}, Vgg \cite{simonyan2014very}, DenseNet \cite{huang2017densely}. 
Both datasets are balanced with 10 and 100 classes respectively, each having 
$500$ and $5,000$ training images per class.
To meet the larger number of classes than feature dimensions, 
we use $6$ and $64$ as the feature dimensions, respectively.
Then, to obtain different dimensional feature for every backbone, we attach a linear layer after the end of backbone, 
which can transform feature dimensions. 

\textbf{Training}
To make NeurCol appear, we follow \cite{PNAS2020}'s practice. 
More details refer to Appendix.\ref{TrainDetails}. 

\begin{table*}[t]
  \scriptsize
  \centering
  \setlength\tabcolsep{2.5pt}
  \begin{tabular}{c|c|c|c|c|c|c|c|c|c|c}
    \toprule
        Experiment Index & 0 & 1 & 2 & 3 & 4 & 5 & 6 & 7 & 8 & 9 \\ 
        \midrule
        Train CE & 0.0019 & 0.0018 & 0.0019 & 0.0019 & 0.0019 & 0.0018 & 0.0019 & 0.0019 & 0.0019 & 0.0019 \\
        Val CE & 0.586 & 0.6965 & 0.6871 & 0.61 & 0.6585 & 0.576 & 0.5565 & 0.5972 & 0.5892 & 0.6243 \\
        \midrule
        Train ACC & 100.0 & 100.0 & 100.0 & 100.0 & 100.0 & 100.0 & 100.0 & 100.0 & 100.0 & 100.0 \\
        Val ACC & 87.96 & 87.96 & 88.25 & 87.94 & 87.32 & 87.58 & 88.39 & 88.32 & 87.98 & 87.73 \\
        \bottomrule
    \end{tabular}
    \caption{\footnotesize Performance comparison of different Permutation features on CIFAR10 and Vgg11. }
  \label{table1}
\end{table*}

\begin{table*}[t]
  \scriptsize
  \centering
  \setlength\tabcolsep{2.5pt}
  \begin{tabular}{c|c|c|c|c|c|c|c|c}
    \toprule
        Equivalence & Dataset & Model & Std Train CE & Std Val CE & Range Val CE & Std Train ACC & Std Val ACC & Range Val ACC \\ 
        \midrule
        \multirow{6}*{ Permutation }  & \multirow{3}*{ CIFAR10 } 
        & VGG11 & 0.0 & 0.01 & 0.034 & 0.0 & 0.209 & 0.71 \\
        & & Resnet34 & 0.0 & 0.018 & 0.058 & 0.0 & 0.197 & 0.71 \\
        & & DenseNet121 & 0.0 & 0.011 & 0.04 & 0.0 & 0.157 & 0.49 \\
        \cmidrule{2-9}
         & \multirow{3}*{ CIFAR100 } 
        & VGG11 & 0.0 & 0.027 & 0.084 & 0.004 & 0.353 & 1.38 \\
        & & Resnet34 & 0.0 & 0.042 & 0.134 & 0.005 & 0.468 & 1.37 \\
        & & DenseNet121 & 0.0 & 0.012 & 0.037 & 0.003 & 0.232 & 0.84 \\
        \midrule
        \multirow{6}*{ Rotation }  & \multirow{3}*{ CIFAR10 } 
        & VGG11 & 0.0 & 0.045 & 0.14 & 0.0 & 0.317 & 1.07 \\
        & & Resnet34 & 0.0 & 0.078 & 0.23 & 0.0 & 0.538 & 1.73 \\
        & & DenseNet121 & 0.0 & 0.059 & 0.187 & 0.0 & 0.18 & 0.63 \\
        \cmidrule{2-9}
         & \multirow{3}*{ CIFAR100 } 
        & VGG11 & 0.0 & 0.034 & 0.099 & 0.0 & 0.455 & 1.52 \\
        & & Resnet34 & 0.0 & 0.036 & 0.136 & 0.005 & 0.403 & 1.16 \\
        & & DenseNet121 & 0.0 & 0.024 & 0.081 & 0.004 & 0.516 & 1.59 \\
        \bottomrule
    \end{tabular}
    \caption{\footnotesize More comprehensive results on different datasets and models. 
    Std indicates the standard deviation of ten metrics and Range indicates the maximal metric minus minimal. } 
  \label{table2}
\end{table*}

\textbf{Results}
Table.\ref{table1} presents results on CIFAR10 and Vgg11 with different Permutation. 
All metrics are given when the model converges to NeurCol ($100\%$ accuracy and zero loss on training set). 
We observe that, even though all experiments achieved $0$ cross-entropy loss and $100\%$ accuracy, 
they still exhibit significant differences in test loss and accuracy.
This implies although permutation hardly affects optimization, it has a significant impact on generalization.
Table.\ref{table2} provides results of different dataset, backbone and two transformations 
to reveal the same phenomenon.
These experimental results answer the question we posed in Section.\ref{motivation}, 
demonstrating that different feature order and direction can influence generalization performance of models. 

\subsection{Analysis for Permutation} 
We aim to explain why \textit{Grassmannian Frame} with 
different Permutation can lead to different generalization performance theoretically. 
Most of symbol definitions are adopted from Section.\ref{GeneralizationPerspective}.
\begin{restatable}[]{theorem}{maintheoremtwo}\label{main_theorem}
  Given the dataset $S$ and a classifier $(\boldsymbol{M}, f(\cdot;\boldsymbol{w}))$, assume
  the classifier has already achieved NeurCol with maximal norm $\rho$. 
  Suppose $(\boldsymbol{M}, f(\cdot; \boldsymbol{w}))$ can linearly separate $S$ by margin $\{\gamma_{i,j}\}_{i \neq j}$. 
  Besides, we make the following assumptions: 
  \begin{itemize}
    \item $f(\cdot, \boldsymbol{w})$ is $L$-Lipschitz for any $\boldsymbol{w}$, $i.e.$ 
      $\forall \boldsymbol{x}_1, \boldsymbol{x}_2$, $\Vert f(\boldsymbol{x}_1, \boldsymbol{w}) - f(\boldsymbol{x}_2, \boldsymbol{w}) 
      \Vert \le L \Vert \boldsymbol{x}_1 - \boldsymbol{x}_2\Vert$. 
    \item $S$ is large enough such that $N_i \ge \max_{j \neq i} \mathcal{N}(\frac{\gamma_{i,j}}{L\Vert M_i - M_j \Vert}, \mathcal{M}_i)$ for every class $i$. 
    \item label distribution and labels of $S$ are balanced, $i.e.$ $p(i) = \frac{1}{C}$ and $N_i = \frac{N}{C}, \forall i \in [C]$
    \item $S_{i}$ is drawn from probability $\mathcal{P}_{\boldsymbol{x}|y=i}$, where probability's tight support is denoted as $\mathcal{M}_i$. 
  \end{itemize}
  where $\mathcal{N}(\cdot, \mathcal{M}_{i})$ is the covering number of $\mathcal{M}_{i}$, refer to Appendix.\ref{coveringnumberproof} for its definition. 
  Then expected accuracy of $(\boldsymbol{M}, f(\cdot; \boldsymbol{w}))$ over the entire distribution is given by 
  \begin{equation*}
  \begin{aligned}
    \mathbb{P}_{\boldsymbol{x},y}\Big(\max_{c} [M f(\boldsymbol{x};\boldsymbol{w})]_{c} = y\Big) > 
    1 - \frac{1}{2 N} \sum_{i=1}^{C} \max_{j \neq i} \mathcal{N}( \frac{1}{L} \sqrt{\frac{\rho^{2} - M_i^T M_j}{2}}, \mathcal{M}_i)
  \end{aligned}
  \end{equation*}
\end{restatable} 

\begin{remark}
  Permutation transformation auctually changes the order of features of every class. 
  Given the \textit{Grassmannian Frame} $\{M_i\}_{i=1}^{C}$, we denote the \textit{Type II Equivalent} frame of permutation $\pi$
  as $\{M_{\pi(i)}\}_{i=1}^{C}$. 
  Therefore, the covering number in our theorem becomes $\mathcal{N}( \frac{1}{L} \sqrt{\frac{\rho^{2} - M_{\pi(i)}^T M_{\pi(j)}}{2}}, \mathcal{M}_i)$, 
  which leads to different accuracy bound. 
\end{remark}

\subsection{Insight and Discussion} 

\begin{figure}[htbp]	
  \centering
	\subfigure[Dog is close to Cat and Fire Truck is close to Car.]
	{
		\begin{minipage}{6.5cm}\label{Fig1a}
			\centering          
			\includegraphics[scale=0.23]{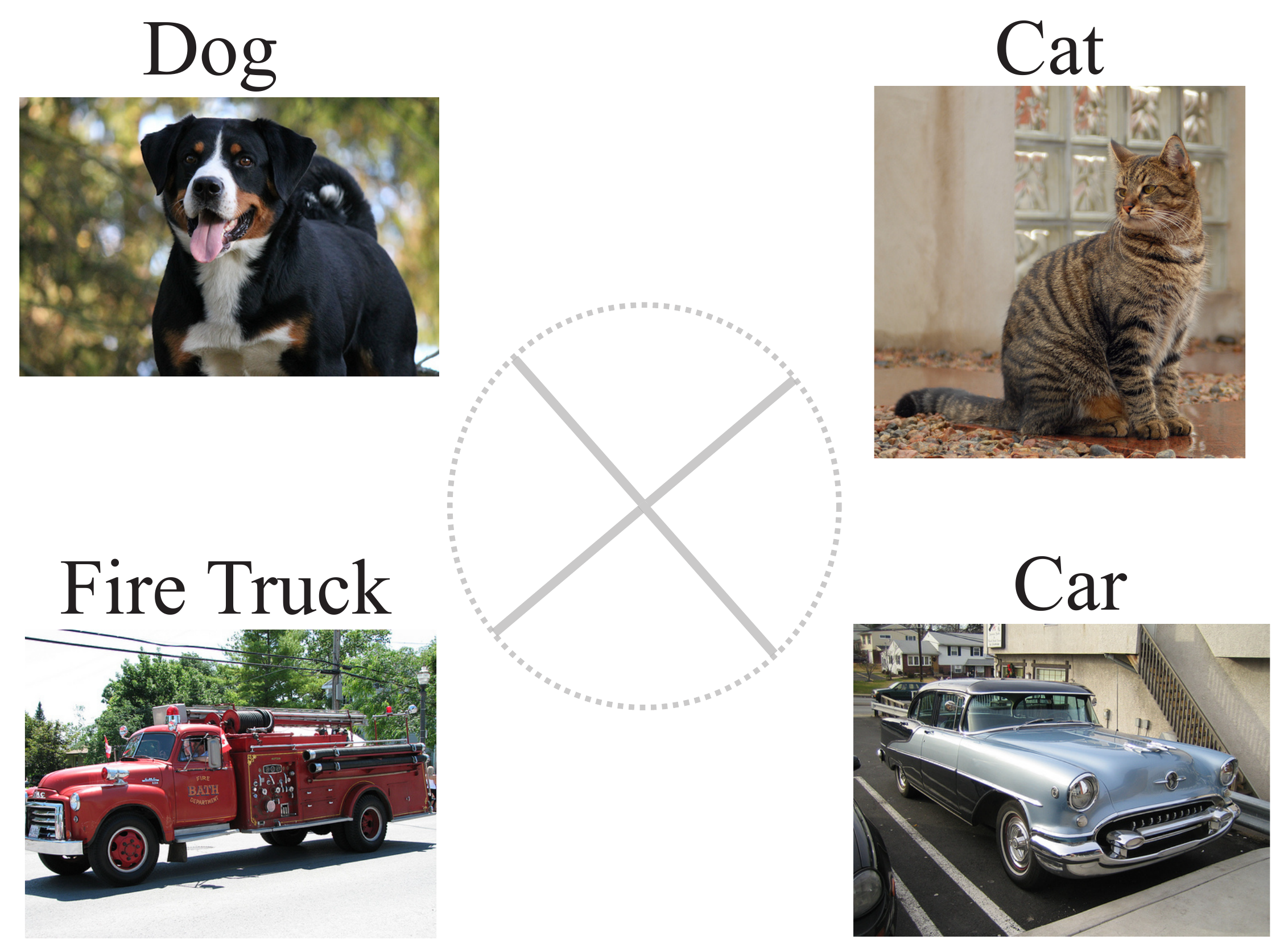}
		\end{minipage}
	}
	\subfigure[Swap Dog and Fire Truck from left.] 
	{
		\begin{minipage}{6.5cm}\label{Fig1b}
			\centering      
			\includegraphics[scale=0.23]{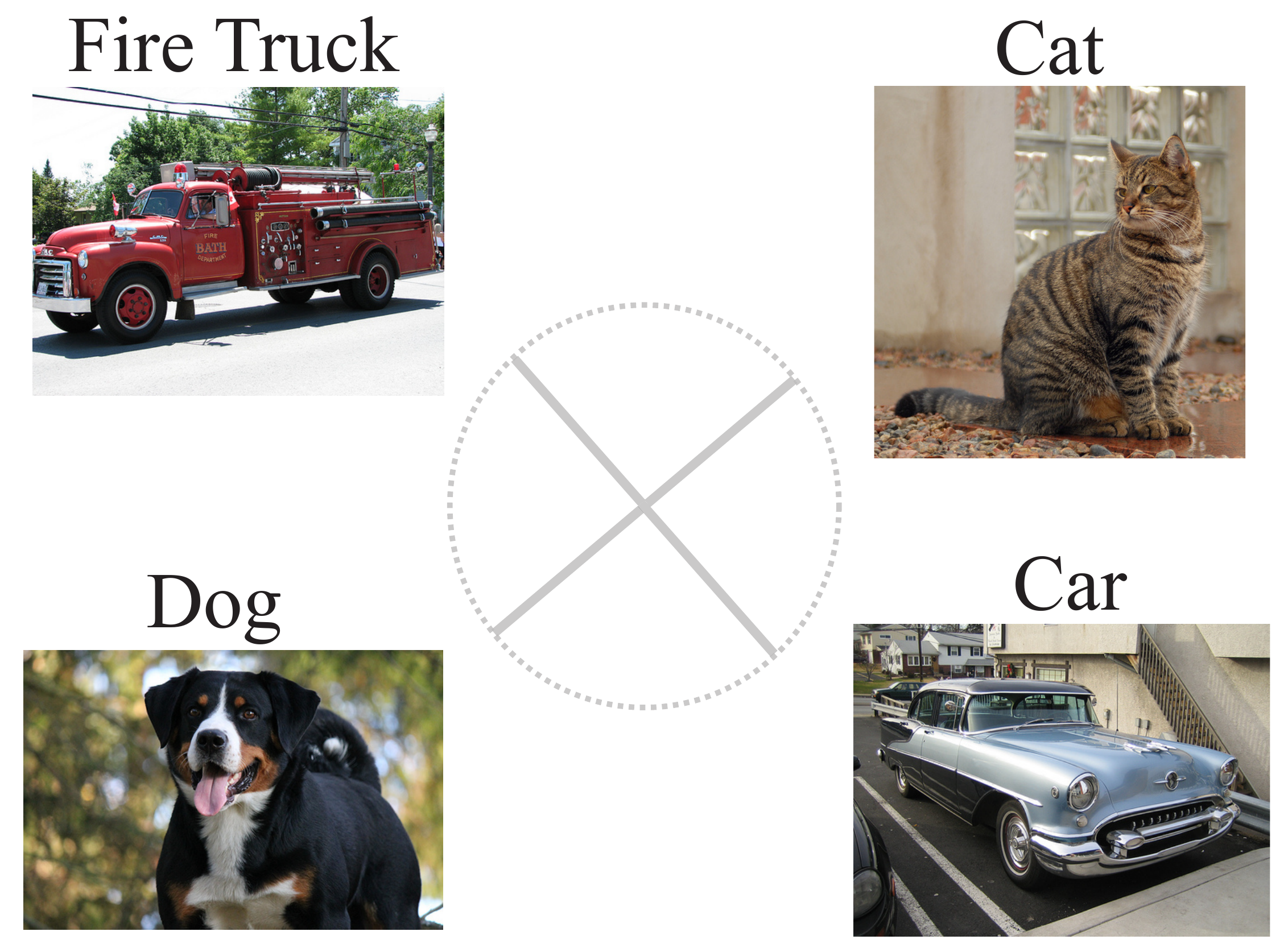}
		\end{minipage}
	}
	\caption{The \textit{Grassmannian Frame} feature alignment of four classes in $\mathbb{R}^{2}$, 
  including Cat, Car, Dog and Fire Truck. 
  All images are from ImageNet \cite{deng2009imagenet}.}
\end{figure}

\textbf{Permutation}
We provide an example to give the insight of \textit{\textbf{Symmetric Generalization}} of Permutation. 
Consider a \textit{Grassmannian Frame} $\{M_i\}_{i=1}^{4}$ living in $\mathbb{R}^2$, 
it resembles a cross (Theorem III.1 of \cite{GeometricPropertiesofGrassmannian}). 
As shown in Figure.\ref{Fig1a}, there are four classes that correspond to different feature $M_i$. 
Obviously, since Dog and Cat look similar to each other, they deserve to have a smaller margin in the feature space (near to each other). 
The same goes for the other two categories. 
However, if we swap the features of Truck and Dog to increase the distance between 
dogs and cats, as shown in Figure.\ref{Fig1b}, the semantic relationship in the feature 
space would be disrupted. We argue that this can harm the model's training and 
lead to worse generalization.

\textbf{Rotation} 
\textit{\textbf{Symmetric Generalization}} of Rotation has been completely beyond our initial expectation. 
We believe that the margin or correlation between features is the most effective tool to understand NeurCol phenomenon. 
However, it fails to explain why different directions of features have different generalization. 
In Deep Learning community, Implicit Bias phenomenon is a possible way to approach this finding. 
\cite{soudry2018the} has proved that the Gradient Descent optimization leads to 
weight direction that maximizes the margin when using logistic loss and linear models. 
As a further progress, \cite{Lyu2020Gradient} extends this result to homogeneous neural networks. 
We speculate that the explanations for \textit{\textbf{Symmetric Generalization}} of Rotation may be 
hidden within the layers of neural networks.
Therefore, studying the Implicit Bias of 
deep models layer by layer could be a promising direction for future research. 
This is beyond the scope of our current work, and we leave it as a topic for our future work.

\section{Conclusion}
In this paper, we justify \textit{Generalized Neural Collapse} hypothesis 
by leading to \textit{Grassmannian Frame} into classification problems, 
which does not assume specific numbers of classes feature dimensions, and every two vectors in it can achieve maximal distances on a sphere. 
In addition, awaring that \textit{Grassmannian Frame} is symmetric geometrically, we propose a question: 
is generalization of the model invariable to symmetric transformations of \textit{Grassmannian Frame}? 
To explore this question, we conduct a series of empirical and theoretical analysis, 
and finally find \textit{Symmetric Generalization} phenomenon. 
This phenomenon suggests that the generalization of a model is influenced by geometric 
invariant transformations of the \textit{Grassmannian Frame}, including $Permutation(C)$ and $SO(d)$.

\newpage

\bibliographystyle{unsrtnat}
{
\small
\bibliography{main}
}

\newpage
\appendix

\section{Related Work} 

Neural Collapse (NeurCol) was first observed by \citet{PNAS2020} in $2020$., and it 
has since sparked numerous studies investigating deep classification models. 

Many of these studies \cite{ji2022an, fang2021exploring, zhu2021geometric, zhou2022all, NEURIPS2022_4b3cc0d1} 
have focused on the optimization aspect of NeurCol, proposing various 
optimization models and analyzing them.
For example, \citet{zhu2021geometric} proposed Unconstrained Feature Models and provided the first global optimization analysis of NeurCol, while 
\citet{fang2021exploring} proposed Layer-Peeled Model, a nonconvex yet analytically tractable 
optimization program, to prove NeurCol and predict \textit{Minority Collapse}, an imbalanced version of NeurCol. 
Other studies have explored NeurCol under Mean Square Error Loss (MSE) 
\cite{han2022neural, tirer2022extended, zhou2022optimization, mixon2020neural, poggio2020explicit}. 
For instance, \citet{han2022neural, tirer2022extended} justified NeurCol under the MSE loss. 
In addition to MSE loss, \cite{zhou2022all} extended such results and analyses
a broad family of loss functions including commonly used label smoothing and focal loss. 
Besides, NeurCol phenomenon also inspires design of loss function in imbalanced 
learning \cite{Xie2022NeuralCI} and Few-Shot Learning \cite{yang2023neural}. 

Previous studies on NeurCol have generally assumed that the number of classes 
is less than the feature dimension, and this assumption has not been questioned 
until recently. In the 11th International Conference on Learning Representations (ICLR 2023), 
\citet{liu2023generalizing} proposed the \textit{Generalized Neural Collapse} hypothesis, which 
removes this restriction by extending the ETF structure to \textit{Hyperspherical Uniformity}. 
In this paper, we contribute to this area by proving the \textit{Generalized Neural Collapse} 
hypothesis and introducing the Grassmannian Frame to better understand the NeurCol phenomenon. 

\section{Proof of Theorem.\ref{communicationperspective}}

We introduce this lemma:
\begin{lemma}\label{guass_geo}
  Suppose that $\boldsymbol{0} \notin \mathcal{K}$, and that $\mathcal{K}$ is a closed set. 
  Suppose $\boldsymbol{g} \sim \mathcal{N}(\boldsymbol{0}, \sigma^2 \boldsymbol{I})$, we have:
  $$
  -\sigma^2 \log \mathbb{P}_{\boldsymbol{g}}(\mathcal{K}) \rightarrow \min_{\boldsymbol{g} \in \mathcal{K}} \{\frac{1}{2}\Vert \boldsymbol{g} \Vert^2 \}
  \ \ \text{as} \ \  \sigma \rightarrow 0
  $$
\end{lemma}
This lemma can establish the relationship between geometry and probability. Then we start our proof.

\communicationperspective*

\begin{proof}
  First, let us think about what kind of decoding is optimal.
  According to \cite{shannon1959probability}, since the gaussian density is monotone decreasing with distance, 
  an optimal decoding system for gaussian channel is the minimum distance decoding, $i.e.$
  \begin{equation*}
  \begin{aligned}
    \hat{c} = \underset{c \in \{1, \dots, C\}}{\arg \min} \Vert M_c - h \Vert
  \end{aligned}
  \end{equation*}
  where $\hat{c}$ is the prediction result.
  we consider the two-classes communication problem: there is only two number $c$ and $c'$ transmitted. 
  We denote the event that a $c$ signal is recover as $c'$ as $\varepsilon_{c'|c}$,
  then 
  \begin{equation*}
  \begin{aligned}
    \varepsilon_{c'|c} = \{ \boldsymbol{g} \in \mathbb{R}^{d}, \Vert \boldsymbol{g} - M_c \Vert > \Vert \boldsymbol{g} - M_{c'} \Vert\}
  \end{aligned}
  \end{equation*}
  According to Lemma.\ref{guass_geo}, we have
  \begin{equation*}
  \begin{aligned}
    -\sigma^2 \log \mathbb{P}_{\boldsymbol{g}}(\varepsilon_{c'|c}) \rightarrow \min_{\boldsymbol{g} \in \varepsilon_{c'|c} } \{\frac{1}{2}\Vert \boldsymbol{g} \Vert^2 \}=\frac{1}{8} \Vert M_c - M_{c'} \Vert^2, \sigma \rightarrow 0.
  \end{aligned}
  \end{equation*}
  For all numbers transmitted, the error event $\varepsilon$ could be devided into the error event between every two numbers, $i.e.$
  \begin{equation*}
  \begin{aligned}
    \varepsilon = \bigcup_{c \neq c'}  \varepsilon_{c|c'}
  \end{aligned}
  \end{equation*}
  So 
  \begin{equation*}
  \begin{aligned}
    -\sigma^2 \log \mathbb{P}_{\boldsymbol{g}}(\varepsilon) \rightarrow \frac{1}{8} \min_{c \neq c'} \Vert M_c - M_{c'} \Vert^2, \sigma \rightarrow 0.
  \end{aligned}
  \end{equation*}
  To obtain the code with minimal error probability, we can maximize $\min_{c \neq c'} \Vert M_c - M_{c'} \Vert^2$. 
  With a norm constraint for every code, we have 
  \begin{equation*}
  \begin{aligned}
    \max_{\forall c, \Vert M_c \Vert = 1} \min_{c \neq c'} \Vert M_c - M_{c'} \Vert^2 
    \Leftrightarrow 
    \min_{\forall c, \Vert M_c \Vert = 1} \max_{c \neq c'} \langle M_c, M_{c'} \rangle
  \end{aligned}
  \end{equation*}
\end{proof}

\section{Proof of Theorem.\ref{CE_grass}}

In this section, we prove Theorem.\ref{CE_grass}.
Here are two lemmas that we are going to use. 

\subsection{Lemmas}

\begin{lemma}[\textbf{Lemma.7 of \cite{yang2021learning}: Lipschitz Properties of Softmax}]\label{LipschitzSoftmax}
  Given $x \in \mathbb{R}^{C}$, the function $Softmax(x)$ is defined as 
  \begin{equation*}
  \begin{aligned}
    Softmax(x) = \left[ \frac{e^{x_1}}{\sum_{i=1}^{C} e^{x_i}}, \dots, \frac{e^{x_C}}{\sum_{i=1}^{C} e^{x_i}} \right]^{T},
  \end{aligned}
  \end{equation*}
  then the function $Softmax(\cdot)$ is $\sqrt{\frac{C}{2}}$-Lipschitz continuous.
\end{lemma}

\begin{lemma}\label{determinant}
  Given any matrix $\boldsymbol{A} \in \mathbb{R}^{n \times n}$ with constant $a$ on the diagonals and constant $c$ on off diagonals, $i.e.$
  \begin{equation*}
    \boldsymbol{A}=
    \left[
      \begin{array}{ccccccc}
        a & c & \ldots  & c \\
        c & a & \cdots  & c \\
        \vdots & \vdots & \ddots & \vdots \\
        c & c & \ldots & a \\
      \end{array}
    \right],
  \end{equation*}
  then we have $|\boldsymbol{A}| = \left(a-c\right)^{n-1} \left(a + (n-1)c\right)$.
\end{lemma}

\begin{proof}
  \begin{equation*}
    |\boldsymbol{A}| 
    \xlongequal{\textbf{step1}}
    \left|
      \begin{array}{ccccccc}
        a-c & 0 & \ldots  & c-a \\
        0 & a-c & \cdots  & c-a \\
        \vdots & \vdots & \ddots & \vdots \\
        c & c & \ldots & a \\
      \end{array}
    \right|
    \xlongequal{\textbf{step2}}
    \left|
      \begin{array}{ccccccc}
        a-c & 0 & \ldots  & 0 \\
        0 & a-c & \cdots  & 0 \\
        \vdots & \vdots & \ddots & \vdots \\
        c & c & \ldots & a+(n-1)c \\
      \end{array}
    \right|
  \end{equation*}

  \begin{itemize}
    \item[] \textbf{step1}: Subtract the first $n-1$ rows by the last row. 
    \item[] \textbf{step2}: Add the first $n-1$ columns on the last column. 
  \end{itemize}
\end{proof}

\subsection{Background}
In the proof of Theorem.\ref{CE_grass}, the techniques that we used is mainly limit analysis, and in addition a little linear algebra. 

\subsubsection{Assumption}
We make the following assumption: 
\begin{assumption}\label{assumption_}
  In update of (\ref{GDupdate}), parameters is properly selected: 
  $\lambda / \omega = \alpha / \beta = N / C$. In addition, $\alpha$ 
  is small enough to assure the system would converge and the norms of every vector in 
  $\boldsymbol{M}, \boldsymbol{Z}$ is always bounded by $\rho$. 
\end{assumption}
This small enough $\alpha$ is necessary and reasonable for the convergce of Gradient Descent. 
With this assumption, we can assure that every variables in system (\ref{GDupdate}) 
would not blow up and finally converge to a stable state. 
And the condition $\lambda / \omega = \alpha / \beta = N / C$ is to make sure 
both classifiers and features can be bounded by the same maximal $\ell_2$ norm. 

\subsubsection{Symbol Regulations}
We regulate our symbols for clearer representation. 
Recall our setting, we use $\boldsymbol{z}_{y, i}$ to denote the $i$-th smaple in class $y$ and 
every class has $N/C$ samples. 
Here, we put $i$-th sample in all class together, denote it as $\boldsymbol{Z}_{i}$, $i.e.$
\begin{equation*}
\begin{aligned}
  \boldsymbol{Z}_{i} = \left[\boldsymbol{z}_{1, i}, \cdots, \boldsymbol{z}_{C, i}\right] \in \mathbb{R}^{d \times C} 
  \ \ \text{and} \ \ 
  \boldsymbol{Z} = [\boldsymbol{Z}_{1}, \cdots, \boldsymbol{Z}_{N/C}] \in \mathbb{R}^{d \times N}
\end{aligned}
\end{equation*}
Then we denote confidence probability of $\boldsymbol{Z}_i$ given by classifier $\boldsymbol{M}$ as 
\begin{equation*}
\begin{aligned}
  \boldsymbol{P}_{i} = 
  \left[Softmax( \boldsymbol{z}_{1,i}^{T} M), \cdots, Softmax( \boldsymbol{z}_{C,i}^{T} M)\right] \in \mathbb{R}^{C \times C}
\end{aligned}
\end{equation*}
where $Softmax(\cdot)$ transform the $logit$ into a probability vector. 
Refer to Lemma.\ref{LipschitzSoftmax} for definition of $Softmax(\cdot)$. 

\subsubsection{Proof Sketch}
We prove Theorem.\ref{CE_grass} by proving three Lemmas:

\begin{restatable}[\textbf{Variability within Classes}]{lemma}{VariabilitywithinClasses}\label{VariabilityCollapse}
  Consider the update of Gradient Descent (\ref{opt2}), under Assumption.\ref{assumption_}, 
  features of any sample that belong to the same class would converge to the same, $i.e.$
  $\forall i, j \in [N/C], \Vert \boldsymbol{Z}_i^{(t)} - \boldsymbol{Z}_j^{(t)} \Vert \rightarrow 0 \ \text{as} \ t \rightarrow \infty$.
\end{restatable}

\begin{restatable}[\textbf{Convergence to Self Duality}]{lemma}{ConvergencetoSelfDuality}\label{SelfDuality}
  Consider the update of Gradient Descent (\ref{opt2}), 
  under Assumption.\ref{assumption_}, 
  features of any sample would converge to the classifier corresponding to it's category, $i.e.$
  $\forall i \in [N/C], \Vert \boldsymbol{Z}_i^{(t)} - \boldsymbol{M}^{(t)} \Vert \rightarrow 0 \  \text{as} \  t \rightarrow \infty$.
\end{restatable}

\begin{restatable}[\textbf{Convergence to \textit{Grassmannian Frame}}]{lemma}{ConvergencetoGrassmannianFrame}\label{ConvergencetoFrame}
  Consider the function (\ref{opt0}), 
  given any sequence $\{(M^{(t)}, Z^{(t)})\}$, 
  if $\text{CELoss}(M^{(t)}, Z^{(t)}, Y) \rightarrow 0$ as $t \rightarrow \infty$, 
  then $\boldsymbol{M}^{(t)}$ and $\boldsymbol{Z}^{(t)}$ would converge to the solution of 
  $$
    \underset{\boldsymbol{M}, \boldsymbol{Z}}{\max} \  \underset{y \neq y'}{\min} \underset{i \in [N/C]}{\min} \langle M_{y} - M_{y'} , \boldsymbol{z}_{y, i} \rangle
  $$
\end{restatable}

Obviously, Lemma.\ref{VariabilityCollapse} and Lemma.\ref{SelfDuality} can directly indicate \textbf{\textit{NC1}} and \textbf{\textit{NC2}}. 
We first prove them, which show the aggregation effect of cross entropy loss: 
features / linear classifier with the same class converge to the same point. 
Then \textbf{\textit{NC4}} is obvious under the conclusions of Lemma.\ref{VariabilityCollapse} and Lemma.\ref{SelfDuality}. 
Lemma.\ref{ConvergencetoFrame} is the key step to prove that (\ref{opt0}) converges to the min-max optimization as 
$\mathcal{L}(\boldsymbol{M}, \boldsymbol{Z})$ converge to zero. 
However, we still need to combine all Lemmas to obtain the \textbf{\textit{minimized maximal correlation}} characteristic. 
Here is the proof: 
\begin{proof}[Proof of \textbf{(NC3)}]
  First, we know 
  \begin{equation*}
  \begin{aligned}
    \underbrace{
      \min_{\boldsymbol{Z}, \boldsymbol{M}} \text{CELoss}(\boldsymbol{M}, \boldsymbol{Z})
    }_{(\ref{opt0})}
    \Leftrightarrow &
    \min_{\rho > 0} 
    \underbrace{
      \min_{\Vert \boldsymbol{z}_{y, i}\Vert, \Vert M_i \Vert \le \rho} \text{CELoss}(\boldsymbol{M}, \boldsymbol{Z})
    }_{(\ref{opt0}) \  s.t. (\ref{opt1})}
    \Leftrightarrow
    \min_{\lambda, \omega > 0} 
    \underbrace{
      \min_{\boldsymbol{Z}, \boldsymbol{M}} \mathcal{L}(\boldsymbol{M}, \boldsymbol{Z})
    }_{(\ref{opt2})},
  \end{aligned}
  \end{equation*}
  while Lemma.\ref{ConvergencetoFrame} establishes a bridge between (\ref{opt0}) and the following max-min problem:
  \begin{equation*}
  \begin{aligned}
    \underbrace{
      \min_{\boldsymbol{Z}, \boldsymbol{M}} \text{CELoss}(\boldsymbol{M}, \boldsymbol{Z})
    }_{(\ref{opt0})}
    \Leftrightarrow
    \max_{\rho > 0} 
    \underbrace{
      \max_{\Vert \boldsymbol{z}_{y, i}\Vert, \Vert M_i \Vert \le \rho} \underset{y \neq y'}{\min} \underset{i \in [N/C]}{\min} \langle M_{y} - M_{y'} , \boldsymbol{z}_{y, i} \rangle
    }_{\textbf{max-min}}
  \end{aligned}
  \end{equation*}
  According to Lemma.\ref{VariabilityCollapse} and \ref{SelfDuality}, 
  we know the solution of (\ref{opt2}) converge to 
  \begin{equation}\label{c1}
  \begin{aligned}
    \forall y \in [C], \forall i \in [N/C], \boldsymbol{z}_{y, i} = M_y
  \end{aligned},
  \end{equation}
  and there must exist a $\rho(\lambda, \omega)$ such that
  \begin{equation}\label{c2}
  \begin{aligned}
    \forall y \in [C], \forall i \in [N/C], \Vert \boldsymbol{z}_{y, i} \Vert = \Vert M_y \Vert = \rho(\lambda, \omega)
  \end{aligned}.
  \end{equation}
  Therefore, we substitute (\ref{c1}) and (\ref{c2}) into \textbf{max-min}:
  \begin{equation*}
    \begin{aligned}
      &  
        \max_{\Vert \boldsymbol{z}_{y, i}\Vert, \Vert M_i \Vert \le \rho} \underset{y \neq y'}{\min} \underset{i \in [N/C]}{\min} \langle M_{y} - M_{y'} , \boldsymbol{z}_{y, i} \rangle
      \\ \Leftrightarrow &
      \max_{\Vert M_i \Vert = \rho} \underset{y \neq y'}{\min} \langle M_{y} - M_{y'} , M_{y} \rangle
      \\ \Leftrightarrow &
      \max_{\Vert M_i \Vert = \rho} \underset{y \neq y'}{\min} \left( \rho^{2} - \langle M_{y'} , M_{y} \rangle \right)
      \\ \Leftrightarrow &
      \underbrace{
        \min_{\Vert M_i \Vert = \rho} \underset{y \neq y'}{\max} \langle M_{y'} , M_{y} \rangle
      }_{\textbf{min-max}}
    \end{aligned}
    \end{equation*}
    where $\Leftrightarrow$ symbol in above equation means the solution of these optimization problems would converge to the same. 
    \textbf{min-max} is exactly our expectant \textbf{\textit{minimized maximal correlation}} characteristic. 
    To ensure the establishment of above equivalence, $\rho \rightarrow \infty$ is required to
    meet $\text{CELoss}(\boldsymbol{Z}, \boldsymbol{M}) \rightarrow 0$. 
\end{proof}

\subsubsection{Gradient Calculation} 
Before starting our proof, we calculate the gradient of (\ref{opt2}) in terms of features $\boldsymbol{Z}$ and classifiers $\boldsymbol{M}$:
\begin{equation*}
\begin{aligned}
  \forall y \in [C], \forall i \in [N/C], \nabla_{\boldsymbol{z}_{y, i}} \mathcal{L}(\boldsymbol{Z}, \boldsymbol{M}) 
  = -M_y + \sum_{y'=1}^{C} \left[ Softmax(\boldsymbol{z}_{y, i}^{T} \boldsymbol{M}) \right]_{y'} M_{y'} + \omega \boldsymbol{z}_{y, i} 
  \\
  \forall y \in [C], \nabla_{M_y} \mathcal{L}(\boldsymbol{Z}, \boldsymbol{M}) = 
  -\sum_{i=1}^{N/C} \boldsymbol{z}_{y, i} + \sum_{y'=1}^{C} \sum_{i=1}^{N/C} 
  \left[ Softmax(\boldsymbol{z}_{y', i}^{T} \boldsymbol{M}) \right]_{y} \boldsymbol{z}_{y', i} + \lambda M_y
\end{aligned}
\end{equation*}
Then we turn it into the matrix form by arranging $y=1, \cdots, C$ in every column.
\begin{equation}\label{gradient}
\begin{aligned}
  \forall i \in [N/C], & \nabla_{\boldsymbol{Z}_i} \mathcal{L} (\boldsymbol{Z}, \boldsymbol{M}) = 
  - \boldsymbol{M} + \boldsymbol{M} \boldsymbol{P}_{i} + \omega \boldsymbol{Z}_{i}
  \\
  & \nabla_{\boldsymbol{M}} \mathcal{L} (\boldsymbol{Z}, \boldsymbol{M}) = 
  - \sum_{i=1}^{N/C} \boldsymbol{Z}_{i} + \sum_{i=1}^{N/C} \boldsymbol{Z}_{i} \boldsymbol{P}_{i}^{T} + \lambda \boldsymbol{M}   
\end{aligned}
\end{equation}

\subsection{Proof of Lemma.\ref{VariabilityCollapse}}
\VariabilitywithinClasses*

\begin{proof}
  With the conclusion of Lemma.\ref{SelfDuality}, we can easily prove Lemma.\ref{VariabilityCollapse}. 
  For any $i, j \in [N/C]$, if $t \rightarrow \infty$, we have
  \begin{equation*}
  \begin{aligned}
    \Vert \boldsymbol{Z}_{i}^{(t)} - \boldsymbol{Z}_{j}^{(t)} \Vert = 
    \Vert \boldsymbol{Z}_{i}^{(t)} - \boldsymbol{M}^{(t)} + \boldsymbol{M}^{(t)} - \boldsymbol{Z}_{j}^{(t)} \Vert
    \le \Vert \boldsymbol{Z}_{i}^{(t)} - \boldsymbol{M}^{(t)} \Vert + \Vert \boldsymbol{Z}_{j}^{(t)} - \boldsymbol{M}^{(t)} \Vert
    \rightarrow 0
  \end{aligned}
  \end{equation*}
\end{proof}

\subsection{Proof of Lemma.\ref{SelfDuality}}
\ConvergencetoSelfDuality*
\begin{proof}
  According to the update rule (\ref{GDupdate}) and gradient (\ref{gradient}), for any $i \in [N/C]$, we have 
  \begin{equation*}
  \begin{aligned}
    \boldsymbol{Z}_{i}^{(t+1)} & 
    \leftarrow \boldsymbol{Z}_{i}^{(t)} - \alpha \left[ - \boldsymbol{M}^{(t)} + \boldsymbol{M}^{(t)} \boldsymbol{P}_{i}^{(t)} + \omega \boldsymbol{Z}_{i}^{(t)} \right]
    \\
    \boldsymbol{M}^{(t+1)} & \leftarrow 
    \boldsymbol{M}^{(t)} - \beta \left[ - \sum_{j \in [N/C]} \boldsymbol{Z}_{i}^{(t)} + 
    \sum_{j \in [N/C]} \boldsymbol{Z}_{j}^{(t)} \left(\boldsymbol{P}_{j}^{(t)}\right)^{T} + \lambda \boldsymbol{M}^{(t)} \right]
  \end{aligned}
  \end{equation*}
  Then, we bound $\Delta(t+1, i) = \Big\Vert \boldsymbol{Z}_{i}^{(t+1)} - \boldsymbol{M}^{(t+1)} \Big\Vert$: 
  \begin{equation}\label{long}
  \begin{aligned}
    & \Delta(t+1, i) = \Bigg\Vert 
    \boldsymbol{Z}_{i}^{(t)} - \boldsymbol{M}^{(t)} 
    + 
    \underbrace{
      \left(\alpha \boldsymbol{M}^{(t)} - \beta \sum_{j \in [N/C]} \boldsymbol{Z}_{i}^{(t)} \right) 
    }_{\textbf{(a)}}
    - \\ & 
    \ \ \ \ \ \ \ \ \ \ \ \ \ \ \ \ \ \ \ \ \ \ \ \ \ \ \ \ \ \ \ \ \
    \underbrace{\left(
      \alpha \boldsymbol{M}^{(t)} \boldsymbol{P}_{i}^{(t)}
      -
      \beta \sum_{j \in [N/C]} \boldsymbol{Z}_{j}^{(t)} \left(\boldsymbol{P}_{j}^{(t)}\right)^{T}
    \right)}_{\textbf{(b)}}
    - 
    \underbrace{\left(
      \alpha \omega \boldsymbol{Z}_{i}^{(t)} -
      \beta \lambda \boldsymbol{M}^{(t)}
    \right)}_{\textbf{(c)}}
    \Bigg\Vert
  \end{aligned}
  \end{equation}
  Then we use the assumption that $\alpha / \beta = \lambda / \omega = N / C$, and consider $\textbf{(a)}, \textbf{(b)}$ and $\textbf{(c)}$ separately.
  \begin{equation}\label{rrr0}
  \begin{aligned}
    \textbf{(a)} & = 
    \frac{\alpha C}{N} \sum_{j \in [N/C]} \left(\boldsymbol{M}^{(t)} - \boldsymbol{Z}_j^{(t)}\right) \\
    \textbf{(b)} & = \frac{\alpha C}{N} \sum_{j \in [N/C]} \left( 
      \boldsymbol{M}^{(t)} \boldsymbol{P}_{i}^{(t)}
      -
      \boldsymbol{Z}_{j}^{(t)} \left(\boldsymbol{P}_{j}^{(t)}\right)^{T}
     \right) \\
    \textbf{(c)} & = \alpha \omega
    \left(\boldsymbol{Z}_{i}^{(t)} - \boldsymbol{M}^{(t)}\right)
  \end{aligned}
  \end{equation}
  Then we combine $\textbf{(a)}$ and $\textbf{(b)}$:
  \begin{equation}\label{rrr1}
  \begin{aligned}
    & \textbf{(a)} - \textbf{(b)} \\ = &
    \frac{\alpha C}{N} \sum_{j \in [N/C]} \left[
      \left( \boldsymbol{M}^{(t)} - \boldsymbol{M}^{(t)} \boldsymbol{P}_{i}^{(t)} \right)
      - 
      \left(
        \boldsymbol{Z}_j^{(t)} - \boldsymbol{Z}_{j}^{(t)} \left(\boldsymbol{P}_{j}^{(t)}\right)^{T}
      \right)
      \right]
    \\ = &
    \frac{\alpha C}{N} \sum_{j \in [N/C]} \left[
      \boldsymbol{M}^{(t)} \left( I - \boldsymbol{P}_{i}^{(t)} \right)
    - 
    \boldsymbol{Z}_j^{(t)} \left(
      I - \left(\boldsymbol{P}_{j}^{(t)}\right)^{T}
    \right)
    \right]
    \\ = &
    \frac{\alpha C}{N} \sum_{j \in [N/C]} \Bigg[
    \boldsymbol{M}^{(t)} \left( I - \boldsymbol{P}_{i}^{(t)} \right)
    -
    \boldsymbol{M}^{(t)} \left( I - \left(\boldsymbol{P}_{j}^{(t)}\right)^{T} \right)
    +
    \\ & \ \ \ \ \ \ \ \ \ \ \ \ \ \ \ \ \ \ \ \ \ \ \ \ \ \ \ \ \ \ \ \ \ \ \ \ \ \ \ \ \ \ \ \ \ \ \ \ \ \ \ \ \ \ \ \ \ \ 
    \boldsymbol{M}^{(t)} \left( I - \left(\boldsymbol{P}_{j}^{(t)}\right)^{T} \right)
    - 
    \boldsymbol{Z}_j^{(t)} \left(
      I - \left(\boldsymbol{P}_{j}^{(t)}\right)^{T}
    \right)
    \Bigg]
    \\ = &
    \frac{\alpha C}{N} \sum_{j \in [N/C]} \left[
    \boldsymbol{M}^{(t)} \left( \left(\boldsymbol{P}_{j}^{(t)}\right)^{T} - \boldsymbol{P}_{i}^{(t)} \right)
    +
    \left(\boldsymbol{M}^{(t)} - \boldsymbol{Z}_j^{(t)}\right)
     \left(
      I - \left(\boldsymbol{P}_{j}^{(t)}\right)^{T}
    \right)
    \right]
    \\ = &
    \underbrace{
      \frac{\alpha C}{N} \sum_{j \in [N/C]} 
    \boldsymbol{M}^{(t)} \left( \left(\boldsymbol{P}_{j}^{(t)}\right)^{T} - \boldsymbol{P}_{i}^{(t)} \right)
    }_{\textbf{(A)}}
    +
    \underbrace{
      \frac{\alpha C}{N} \sum_{j \in [N/C]} 
      \left(\boldsymbol{M}^{(t)} - \boldsymbol{Z}_j^{(t)}\right)
      \left(
       I - \left(\boldsymbol{P}_{j}^{(t)}\right)^{T}
     \right)
    }_{\textbf{(B)}}
    \\ = & \textbf{(A)} + \textbf{(B)}
  \end{aligned}
  \end{equation}
  Next, we plug (\ref{rrr0}) and (\ref{rrr1}) into (\ref{long}) to derive 
  \begin{equation*}
  \begin{aligned}
    \Delta(t+1, i) & = \left\Vert 
    \boldsymbol{Z}_{i}^{(t)} - \boldsymbol{M}^{(t)} 
    + \textbf{(a)} - \textbf{(b)} - \textbf{(c)}
    \right\Vert
    = 
    \left\Vert 
    \boldsymbol{Z}_{i}^{(t)} - \boldsymbol{M}^{(t)} 
    - \textbf{(c)}
    + \textbf{(A)} + \textbf{(B)}
    \right\Vert
    \\ & =
    \left\Vert 
    (1 - \alpha \omega)\left( \boldsymbol{Z}_{i}^{(t)} - \boldsymbol{M}^{(t)} \right)
    + \textbf{(A)} + \textbf{(B)}
    \right\Vert
    \le 
    (1 - \alpha \omega) \Delta(t, i) + 
    \left\Vert \textbf{(A)} \right\Vert + \left\Vert \textbf{(B)} \right\Vert
  \end{aligned}
  \end{equation*}
  Then we bound $\Vert \textbf{(A)} \Vert$ and $\Vert \textbf{(B)} \Vert$.
  \begin{equation*}
  \begin{aligned}
    \Vert \textbf{(A)} \Vert 
    \le &
    \frac{\alpha C}{N} \sum_{j \in [N/C]} 
    \left\Vert \boldsymbol{M}^{(t)} \right\Vert
    \left\Vert
    \left(\boldsymbol{P}_{j}^{(t)}\right)^{T} - \boldsymbol{P}_{i}^{(t)}
    \right\Vert
    \\ \le &
    \frac{\alpha C^{3/2} \rho}{N} \sum_{j \in [N/C]}     
    \left\Vert
    \left(\boldsymbol{P}_{j}^{(t)}\right)^{T} - \boldsymbol{P}_{i}^{(t)}
    \right\Vert
    \\ = &
    \frac{\alpha C^{3/2} \rho}{N} \sum_{j \in [N/C]}     
    \left\Vert
    \left(\boldsymbol{P}_{j}^{(t)}\right)^{T} - \boldsymbol{P}_{j}^{(t)} + 
    \boldsymbol{P}_{j}^{(t)}- \boldsymbol{P}_{i}^{(t)}
    \right\Vert
    \\ \le &
    \frac{\alpha C^{3/2} \rho}{N} \sum_{j \in [N/C]}     
    \left(
    \left\Vert
    \underbrace{
      \left(\boldsymbol{P}_{j}^{(t)}\right)^{T} - \boldsymbol{P}_{j}^{(t)}
    }_{\textbf{(A.1)}}
    \right\Vert
     + 
    \left\Vert
    \underbrace{
      \boldsymbol{P}_{j}^{(t)}- \boldsymbol{P}_{i}^{(t)}
    }_{\textbf{A.2}}
    \right\Vert
    \right)
  \end{aligned}
  \end{equation*}
  For \textbf{(A.1)}, we have 
  \begin{equation*}
  \begin{aligned}
    \Vert \textbf{(A.1)} \Vert & = \left\Vert \left(\boldsymbol{P}_{j}^{(t)}\right)^{T} - \boldsymbol{P}_{j}^{(t)} \right\Vert 
    \\ & \le 
    \frac{\sqrt{C}}{\sqrt{2}} \left\Vert
      \left(\boldsymbol{M}^{(t)}\right)^{T} \boldsymbol{Z}_{j}^{(t)} 
      -
      \left( \boldsymbol{Z}_{j}^{(t)} \right)^{T} \boldsymbol{M}^{(t)}
    \right\Vert
    \\ & =
    \frac{\sqrt{C}}{\sqrt{2}} \left\Vert
      \left(\boldsymbol{M}^{(t)}\right)^{T} \boldsymbol{Z}_{j}^{(t)} 
      -
      \left(\boldsymbol{M}^{(t)}\right)^{T} \boldsymbol{M}^{(t)} 
      +
      \left(\boldsymbol{M}^{(t)}\right)^{T} \boldsymbol{M}^{(t)} 
      -
      \left( \boldsymbol{Z}_{j}^{(t)} \right)^{T} \boldsymbol{M}^{(t)}
    \right\Vert
    \\ & \le
    \frac{\sqrt{C}}{\sqrt{2}} 
    \left\Vert \boldsymbol{M}^{(t)} \right\Vert
    \left\Vert
      \boldsymbol{Z}_{j}^{(t)} 
      -
      \boldsymbol{M}^{(t)} 
    \right\Vert
      +
    \frac{\sqrt{C}}{\sqrt{2}} 
    \left\Vert
      \boldsymbol{M}^{(t)}
      -
      \boldsymbol{Z}_{j}^{(t)} 
    \right\Vert
    \left\Vert \boldsymbol{M}^{(t)} \right\Vert
    \\ & =
    \sqrt{2C}
    \left\Vert
      \boldsymbol{M}^{(t)}
      -
      \boldsymbol{Z}_{j}^{(t)} 
    \right\Vert
    \left\Vert \boldsymbol{M}^{(t)} \right\Vert
    \\ & \le
    \sqrt{2} C \rho 
    \Delta(t, j)
  \end{aligned}
  \end{equation*}
  where 
  the first inequality is because the $Softmax(\cdot)$ function is $\sqrt{\frac{C}{2}}$-Lipschitz-continuous,
  and the last inequality follows from the bounded norm of $\boldsymbol{M}$: $\forall t$, we have $\Vert \boldsymbol{M}^{(t)} \Vert \le \sqrt{C} \rho$. 
  For \textbf{(A.2)}, if $i=j$, $\Vert \textbf{(A.2)} \Vert = 0$. If $i \neq j$, we have 
  \begin{equation*}
  \begin{aligned}
    \Vert \textbf{(A.2)} \Vert & = \left\Vert \boldsymbol{P}_{j}^{(t)}- \boldsymbol{P}_{i}^{(t)} \right\Vert
    \\ & \le
    \frac{\sqrt{C}}{\sqrt{2}} \left\Vert
      \left( \boldsymbol{Z}_{j}^{(t)} \right)^{T} \boldsymbol{M}^{(t)} - 
      \left( \boldsymbol{Z}_{i}^{(t)} \right)^{T} \boldsymbol{M}^{(t)}
    \right\Vert
    \\ & \le
    \frac{\sqrt{C}}{\sqrt{2}} \left\Vert
      \boldsymbol{Z}_{j}^{(t)} - 
      \boldsymbol{Z}_{i}^{(t)}
    \right\Vert
    \left\Vert \boldsymbol{M}^{(t)} \right\Vert
    \\ & \le
    \frac{\sqrt{C}}{\sqrt{2}} \left\Vert
      \boldsymbol{Z}_{j}^{(t)} - \boldsymbol{M}^{(t)}
      + \boldsymbol{M}^{(t)} - \boldsymbol{Z}_{i}^{(t)}
    \right\Vert
    \left\Vert \boldsymbol{M}^{(t)} \right\Vert
    \\ & \le
    \frac{C \rho}{\sqrt{2}} \left( \Delta(t, i) + \Delta(t, j) \right)
  \end{aligned}
  \end{equation*}
  For \textbf{(B)}, we have 
  \begin{equation*}
  \begin{aligned}
    \Vert \textbf{(B)} \Vert \le  
    \frac{\alpha C}{N} \sum_{j \in [N/C]} 
    \Delta(t, j)
    \left\Vert
     I - \left(\boldsymbol{P}_{j}^{(t)}\right)^{T}
    \right\Vert
    \le 
    \frac{\alpha C^{2}}{N} \sum_{j \in [N/C]} 
    \Delta(t, j)
  \end{aligned}
  \end{equation*}
  The above final inequality is because 
  both $I$ and $\boldsymbol{P}_{j}^{(t)}$ can be seen as probability simplex, 
  the norm of $I - \boldsymbol{P}_{j}^{(t)}$ is always less than the norm of all-one matrix minus all-zero matrix, 
  where the latter's norm is $C$. Finally, we can bound $\Delta(t+1, i)$:
  \begin{align}
    \Delta(t+1, i) & \le 
    (1-\alpha \omega) \Delta(t, i) + \Vert \textbf{(A)} \Vert + \Vert \textbf{(B)} \Vert
    \notag
    \\ & \le 
    (1-\alpha \omega) \Delta(t, i) + 
    \frac{\alpha C^{3/2} \rho}{N} \sum_{j \in [N/C]}     
    \left(
    \left\Vert \textbf{(A.1)} \right\Vert
     + 
    \left\Vert \textbf{A.2} \right\Vert
    \right)
    + 
    \frac{\alpha C^{2}}{N} \sum_{j \in [N/C]} 
    \Delta(t, j)
    \notag \\ & \le 
    \left(1- \alpha \left( \omega - \frac{C^{2}}{N} \right) \right) \Delta(t, i)
    + 
    \frac{\alpha C^{2}}{N} \sum_{j \neq i } 
    \Delta(t, j)
    + \notag \\ & \ \ \ \ \ \ \ \ \ \ \ 
    \frac{\alpha C^{3/2} \rho}{N} \sum_{j \in [N/C]}
    \sqrt{2} C \rho \Delta(t, j)
    + 
    \frac{\alpha C^{3/2} \rho}{N} \sum_{j \in [N/C]}
    \frac{C \rho}{\sqrt{2}} \left( \Delta(t, i) + \Delta(t, j) \right)
    \notag \\ & \le 
    \left(1- \alpha \left( \omega - \frac{C^{2}}{N} \right) \right) \Delta(t, i)
    + 
    \frac{\alpha C^{2}}{N} \sum_{j \neq i } 
    \Delta(t, j)
    + \notag \\ & \ \ \ \ \ \ \ \ \ \ \ 
    \frac{\sqrt{2} \alpha C^{5/2} \rho^{2}}{N} \sum_{j \in [N/C]}
    \Delta(t, j)
    + 
    \frac{\alpha C^{5/2} \rho^{2}}{N \sqrt{2}} \sum_{j \in [N/C]}
    \left( \Delta(t, i) + \Delta(t, j) \right)
    \notag \\ & \le 
    \left(1- \alpha 
    \underbrace{
      \left( 
        \omega - 
        \frac{C^{2}}{N} - 
        \frac{\sqrt{2}C^{5/2} \rho^{2}}{N} - 
        \frac{C^{3/2} \rho^{2}}{\sqrt{2}} -
        \frac{C^{5/2} \rho^{2}}{N \sqrt{2}}
      \right)
    }_{F1} 
    \right) \Delta(t, i)
    + 
    \notag \\ & \ \ \ \ \ \ \ \ \ \ \ 
    \alpha 
    \underbrace{
      \left( 
        \frac{C^{2}}{N} + 
        \frac{\sqrt{2} C^{5/2} \rho^{2}}{N} + 
        \frac{C^{5/2} \rho^{2}}{N \sqrt{2}}
      \right)
    }_{F2}
    \sum_{j \neq i} \Delta(t, j)
    \label{F1andF2}
  \end{align}
  We put all $\Delta(t+1, \star)$ and $\Delta(t, \star)$ together to derive the difference inequality:
  \begin{equation*}
  \begin{aligned}
    \boldsymbol{\Delta}(t+1) \preceq  \boldsymbol{A} \boldsymbol{\Delta}(t)
  \end{aligned}
  \end{equation*}
  where 
  \begin{equation*}
    \boldsymbol{A}=
    \left[
      \begin{array}{ccccccc}
        1-\alpha F_1 & \alpha F_2 & \ldots  & \alpha F_2 \\
        \alpha F_2 & 1-\alpha F_1 & \cdots  & \alpha F_2 \\
        \vdots & \vdots & \ddots & \vdots \\
        \alpha F_2 & \alpha F_2 & \ldots & 1-\alpha F_1 \\
      \end{array}
    \right]
    \ \text{and} \ 
    \boldsymbol{\Delta}(t) =
    \left[
      \begin{array}{ccccccc}
        & \Delta(t, 1) \\
        & \Delta(t, 2) \\
        & \vdots  \\
        & \Delta(t, N/C) \\
      \end{array}
    \right]
  \end{equation*}
  In above notation, the values of $F_1$ and $F_2$ refer to (\ref{F1andF2}).
  We investigate if $\boldsymbol{\Delta}(t)$ can converge to zero vector by adjusting $\alpha$. 
  According to the Lemma.\ref{determinant}, we know eigen values of $\boldsymbol{A}$ are
  \begin{equation*}
  \begin{aligned}
    & \lambda_{1} = \lambda_{2} = \dots = \lambda_{N/C - 1} = 1 - \alpha \left(F_1 + F_2\right)
    \\
    & \lambda_{N/C} = 1 - \alpha \left(F_1 - \left(N/C - 1\right) F_2\right)
  \end{aligned}
  \end{equation*}
  Here, with properly selected parameters, we can make all eigen values of $\boldsymbol{A}$ in $(-1, 1)$. 
  Therefore, as $t \rightarrow \infty$, $\boldsymbol{A}^{t} \rightarrow \boldsymbol{0}$ and $\boldsymbol{\Delta}(t) \rightarrow \boldsymbol{0}$.
\end{proof}

\subsection{Proof of Lemma.\ref{ConvergencetoFrame}}
\ConvergencetoGrassmannianFrame*
\begin{proof}
  For simplicity, we leave out the upper script $(t)$. 
  First, we have $\forall t$
  \begin{equation}\label{1}
  \begin{aligned}
    \text{CELoss}(\boldsymbol{Z}, \boldsymbol{M}) = &
    \sum_{y=1}^{C} \sum_{i=1}^{N/C} - \log \frac{\exp \big(\langle M_{y}, \boldsymbol{z}_{y, i} \rangle \big)}{\sum_{y'} \exp \big(\langle M_{y'}, \boldsymbol{z}_{y, i} \rangle\big)}
    \\ = &
    \sum_{y=1}^{C} \sum_{i=1}^{N/C} \log \bigg(1 + \sum_{y' \neq y} \exp \big( \langle M_{y'}-M_{y}, \boldsymbol{z}_{y, i} \rangle \big)\bigg) 
    \\ \le &
    \sum_{y=1}^{C} \sum_{i=1}^{N/C} \log \bigg(1 + (C-1) \exp \big(\max_{y' \neq y} \{  \langle M_{y'}-M_{y}, \boldsymbol{z}_{y, i} \rangle \big) \}\bigg)
    \\ \le &
    \frac{N}{C} \sum_{y=1}^{C} \log \bigg(1 + (C-1) \exp \big(\max_{y' \neq y} \max_{i \in [N/C]} \{  \langle M_{y'}-M_{y}, \boldsymbol{z}_{y, i} \rangle \big) \}\bigg)
    \\ \le &
    N \max_{y \in [C]} \log \bigg(1 + (C-1) \exp \big(\max_{y' \neq y} \max_{i \in [N/C]} \{  \langle M_{y'}-M_{y}, \boldsymbol{z}_{y, i} \rangle \big) \}\bigg)
    \\ = &
    N \log \bigg(1 + (C-1) \exp \big(\max_{y \in [C]} \max_{y' \neq y} \max_{i \in [N/C]} \{  \langle M_{y'}-M_{y}, \boldsymbol{z}_{y, i} \rangle \big) \}\bigg)
  \end{aligned}
  \end{equation}
  In addition, we have
  \begin{equation}\label{2}
    \begin{aligned}
    \log \bigg(1 + \exp \big(\max_{y \in [C]} \max_{y' \neq y} 
    \max_{i \in [N/C]} \{  \langle M_{y'}-M_{y}, \boldsymbol{z}_{y, i} \rangle \big) \}
    \bigg)
    \le \text{CELoss} (\boldsymbol{Z}, \boldsymbol{M})
  \end{aligned}
  \end{equation}
  We denote $\max_{y \in [C]} \max_{y' \neq y}$ as $\max_{y' \neq y}$ and define the margin of entire dataset 
  (refer to Section.3.1 of \cite{ji2022an}) as follow:
  $$
    p_{min} := \min_{y \neq y'} \max_{i \in [N/C]} \langle M_y  - M_{y'}, z_{y, i} \rangle
  $$
  Therefore, we have 
  \begin{equation}\label{hah}
  \begin{aligned}
    \underbrace{
      \log \bigg(1 + \exp \left( - p_{min} \right) \bigg)
    }_{\ell_{1}(p_{min})}
    \le 
    \text{CELoss} (\boldsymbol{Z}, \boldsymbol{M})
    \le 
    N \underbrace{
      \log \bigg(1 + (C-1) \exp \left( - p_{min} \right) \bigg)
    }_{\ell_{C-1}(p_{min})}
  \end{aligned}
  \end{equation}
  where $\ell_{a}(p) = \log (1 + a e^{-p})$.
  Then we represent $\ell_{a}(\cdot)$ as the form of exponential function, $i.e.$
  \begin{equation*}
  \begin{aligned}
    \ell_{a}(p) = e^{-\phi_{a}(p)} \ \ \text{and} \ \  \phi_{a}(p) = -\log \log (1 + a e^{-p})
  \end{aligned}.
  \end{equation*}
  Denote 
  the inverse function of $\phi_{a}(\cdot)$ as $\Phi_{a}(\cdot)$, where $\Phi_{a}(p) = -\log (\frac{e^{e^{-p}}-1}{a})$. 
  Then continue from (\ref{hah}), we have 
  \begin{equation*}
  \begin{aligned}
    & \ell_{1}(p_{min}) \le \text{CELoss}(\boldsymbol{Z}, \boldsymbol{M}) \le N \ell_{C-1}(p_{min})
    \\ \Leftrightarrow &
    e^{-\phi_{1}(p_{min})} \le \text{CELoss}(\boldsymbol{Z}, \boldsymbol{M}) \le N e^{-\phi_{C-1}(p_{min})}
    \\ \Leftrightarrow &
    \phi_{C-1}(p_{min}) - \log(N) \le - \log (\text{CELoss}(\boldsymbol{Z}, \boldsymbol{M})) \le \phi_{1}(p_{min})
  \end{aligned}
  \end{equation*}
  According to the monotonicity of $\Phi_{1}(\cdot)$, we have 
  \begin{equation*}
  \begin{aligned}
    \Phi_{1} \left(\phi_{C-1}(p_{min}) - \log(N) \right) \le \Phi_{1} \left( - \log (\text{CELoss}(\boldsymbol{Z}, \boldsymbol{M})) \right) \le p_{min}
  \end{aligned}
  \end{equation*}
  Use the mean value theorem, there exists a $\xi \in (\phi_{C-1}(p_{min}) - \log(N), \phi_{1}(p_{min}))$ such that 
  \begin{equation*}
  \begin{aligned}
    \Phi_{1}(\phi_{C-1}(p_{min}) - \log(N)) =  p_{min} - \Phi_{1}^{'}(\xi) (\phi_{1}(p_{min}) - \phi_{C-1}(p_{min}) + \log(N)),
  \end{aligned}
  \end{equation*}
  then
  \begin{equation}\label{3}
  \begin{aligned}
    p_{min} - 
    \underbrace{\Phi_{1}^{'}(\xi) (\phi_{1}(p_{min}) - \phi_{C-1}(p_{min}) + \log(N))}_{\Delta(t)}
     \le \Phi_{1} \left( - \log (\text{CELoss}(\boldsymbol{Z}, \boldsymbol{M})) \right) \le p_{min}
  \end{aligned}
  \end{equation}
  Then we will show $\Delta(t) = \mathcal{O}(1) (t \rightarrow \infty)$.
  Since 
  \begin{equation*}
  \begin{aligned}
    \xi > \phi_{C-1}(p_{min}) - \log(N) \ge - \log (\text{CELoss}(\boldsymbol{Z}, \boldsymbol{M})) - \log(N)
  \end{aligned},
  \end{equation*}
  we know $\xi \rightarrow \infty$ and $p_{min} \rightarrow \infty$ as $\text{CELoss}(\boldsymbol{Z}, \boldsymbol{M}) \rightarrow 0$. 
  By simple calculation, we have $\phi_{1}(p_{min}) - \phi_{C-1}(p_{min}) \rightarrow \log(C-1)$ 
  and $\Phi_{1}^{'}(\xi) = \frac{e^{e^{-\xi}-\xi}}{e^{e^{-\xi}}-1} \rightarrow 1$. 
  Next, we denote the maximal norm at iteration $t$ as 
  $$
    \rho_{t} = \max_{\boldsymbol{v} \in \boldsymbol{Z}^{(t)} \cup \boldsymbol{M}^{(t)}} \Vert \boldsymbol{v} \Vert.
  $$ 
  Due to $p_{min} \rightarrow \infty$, $\rho_{t} \rightarrow \infty$.
  Then we devide $\rho_{t}$ on the both sides of (\ref{3}) to obtain
  \begin{equation*}
  \begin{aligned}
    \Bigg| \frac{\Phi_{1} \left( - \log (\text{CELoss}(\boldsymbol{Z}, \boldsymbol{M})) \right)}{\rho_{t}} - \frac{p_{min}}{\rho_{t}} \Bigg| \rightarrow 0, t \rightarrow \infty,
  \end{aligned}
  \end{equation*}
  therefore
  \begin{equation*}
  \begin{aligned}
    \frac{\Phi_{1} \left( - \log (\text{CELoss}(\boldsymbol{Z}, \boldsymbol{M})) \right)}{\rho_{t}} \rightarrow \frac{p_{min}}{\rho_{t}}, t \rightarrow \infty
  \end{aligned}.
  \end{equation*}
  So
  \begin{equation*}
  \begin{aligned}
    \min_{\boldsymbol{M}, \boldsymbol{Z}} \text{CELoss}(\boldsymbol{Z},\boldsymbol{M}) \Leftrightarrow & \min_{\rho > 0} \min_{\Vert M_i \Vert, \Vert \boldsymbol{z}_{y, i} \Vert\le \rho} \text{CELoss}(\boldsymbol{Z},\boldsymbol{M})
    \\ \Leftrightarrow & \max_{\rho > 0} \max_{\Vert M_i \Vert, \Vert \boldsymbol{z}_{y, i}\Vert \le \rho} \Phi_{1} \left( - \log (\text{CELoss}(\boldsymbol{Z}, \boldsymbol{M})) \right)
    \\ \Leftrightarrow & \max_{\rho > 0} \max_{\Vert M_i \Vert, \Vert \boldsymbol{z}_{y, i}\Vert \le \rho} \rho \frac{\Phi_{1} \left( - \log (\text{CELoss}(\boldsymbol{Z}, \boldsymbol{M})) \right)}{\rho}
    \\ \Leftrightarrow & \max_{\rho > 0} \max_{\Vert M_i \Vert, \Vert \boldsymbol{z}_{y, i}\Vert \le \rho} \rho \frac{p_{min}}{\rho}
    \\ \Leftrightarrow & \max_{\rho > 0} \max_{\Vert M_i \Vert, \Vert \boldsymbol{z}_{y, i} \Vert \le \rho} \min_{y \neq y'} \min_{i \in [N/C]} \langle M_y  - M_{y'}, \boldsymbol{z}_{y,i} \rangle
  \end{aligned}
  \end{equation*}
  The $\Leftrightarrow$ symbol in above equation means the solution of these optimization problems would converge to the same.
\end{proof}

\section{Proof of Theorem.\ref{main_theorem_0}}\label{ProofofTheorem45}

\begin{theorem}[\textbf{Theorem.5 of \cite{marginBound}: Margin Bound}]\label{margin_bounds}
  Consider a data space $\mathcal{X}$ and a probability measure $\mathcal{P}$ on it. There is a dataset $\{x_{i}\}_{i=1}^{n}$ that contains $n$ samples, 
  which are drawn $i.i.d$ from $\mathcal{P}$.
  Consider an arbitrary function class $\mathcal{F}$ such that $\forall f \in \mathcal{F}$ we have $\sup_{x \in \mathcal{X}} |f(x)| \le K$, 
  then with probability at least $1 - \delta$ over the sample, for all margins $\gamma > 0$ and all $f \in \mathcal{F}$ we have,
  \begin{equation*}
  \begin{aligned}
    \mathbb{P}_{x} (f(x) \le 0) \le 
    \sum_{i = 1}^{n} \frac{\mathbb{I}(f(x_{i}) \le \gamma)}{n} + 
    \frac{\mathfrak{R}_{n}(\mathcal{F})}{\gamma} + 
    \sqrt{\frac{\log(\log_{2} \frac{4K}{\gamma})}{n}} + 
    \sqrt{\frac{\log(1 / \delta)}{2n}}
  \end{aligned}
  \end{equation*}
\end{theorem}

\theoremstyle{Theorem45}
\newtheorem*{Theorem45}{Theorem.\ref{main_theorem_0}}
\begin{Theorem45}[\textbf{Multiclass Margin Bound}]
    Consider a dataset $S$ with $C$ classes. 
    For any classifier $(\boldsymbol{M}, f(\cdot; \boldsymbol{w}))$, 
    we denote its margin between $i$ and $j$ classes as $(M_{i} - M_{j})^{T} f(\cdot; \boldsymbol{w})$.
    And suppose the function space of the margin is $\mathcal{F} = \{ (M_{i} - M_{j})^{T} f(\cdot; \boldsymbol{w}) | \forall i \neq j, \forall \boldsymbol{M}, \boldsymbol{w}\}$, 
    whose uppder bound is 
    \begin{equation*}
    \begin{aligned}
      \sup_{i \neq j} \sup_{\boldsymbol{M}, \boldsymbol{w}} \sup_{x \in \mathcal{M}_i} \left| (M_i - M_j)^{T} f(\boldsymbol{x};\boldsymbol{w}) \right| \le K.
    \end{aligned}
    \end{equation*}
    Then, for any classifier $(\boldsymbol{M}, f(\cdot; \boldsymbol{w}))$ and margins $\{\gamma_{i,j}\}_{i\neq j} (\gamma_{i,j} > 0)$, the following inequality holds with probability at least $1 - \delta$
  \begin{equation*}
    \begin{aligned}
      \mathbb{P}_{x,y}\Big(\max_{c} [M f(\boldsymbol{x};\boldsymbol{w})]_{c} \neq y\Big) 
      \le & 
      \sum_{i=1}^{C} p(i) \sum_{j \neq i} \frac{\mathfrak{R}_{N_i}(\mathcal{F})}{\gamma_{i,j}} + 
      \sum_{i=1}^{C} p(i) \sum_{j \neq i} \sqrt{\frac{\log(\log_{2} \frac{4K}{\gamma_{i,j}})}{N_i}}
      \\ & + \ \text{empirical risk term}\  + \ \text{probability term}
    \end{aligned}
    \end{equation*}
    where
    \begin{equation*}
    \begin{aligned}
      \ \text{empirical risk term}\ & = 
      \sum_{i=1}^{C} p(i) \sum_{j \neq i}
      \sum_{x \in S_i} \frac{\mathbb{I}((M_i - M_j)^{T}f(x) \le \gamma_{i,j})}{N_i}, \\
      \ \text{probability term}\ & =
      \sum_{i=1}^{C} p(i) \sum_{j \neq i} \sqrt{\frac{\log(C(C-1) / \delta)}{2N_i}}.
    \end{aligned}
    \end{equation*}
    $\mathfrak{R}_{N_i}(\mathcal{F})$ is the Rademacher complexity \cite{marginBound, Rademacher} of function space $\mathcal{F}$. 
\end{Theorem45}

\begin{proof}
  First, we decompose the accuracy as accuracies within every class by Bayes Theory:
  \begin{equation}\label{0}
  \begin{aligned}
    \mathbb{P}_{\boldsymbol{x},y}\Big( \underset{c}{\arg \max} [\boldsymbol{M} f(\boldsymbol{x};\boldsymbol{w})]_{c} \neq y \Big) = 
    \sum_{i=1}^{C} p(i) \mathbb{P}_{\boldsymbol{x}|y=i} \Big( \underset{c}{\arg \max} [\boldsymbol{M} f(\boldsymbol{x};\boldsymbol{w})]_{c} \neq y\Big)
  \end{aligned}
  \end{equation}
  where $p(i)$ is the probability density of $i$-th class. Then, we focus on the accuracy within every class $i$. 
  \begin{equation*}
  \begin{aligned}
    \mathbb{P}_{\boldsymbol{x}|y=i} \Big(\underset{c}{\arg \max} [\boldsymbol{M} f(\boldsymbol{x};\boldsymbol{w})]_{c} \neq y\Big) = 
    \mathbb{P}_{\boldsymbol{x}|y=i} \Bigg( \bigcup_{j \neq i} \{ (M_i - M_j)^{T} f(\boldsymbol{x};\boldsymbol{w}) < 0 \} \Bigg)
  \end{aligned}
  \end{equation*}
  According to union bound, we have 
  \begin{equation*}
  \begin{aligned}
    \mathbb{P}_{\boldsymbol{x}|y=i} \Big( \underset{c}{\arg \max} [\boldsymbol{M} f(\boldsymbol{x};\boldsymbol{w})]_{c} \neq y\Big) \le 
    \sum_{j \neq i} \mathbb{P}_{\boldsymbol{x}|y=i} \Big( (M_i - M_j)^{T} f(\boldsymbol{x};\boldsymbol{w}) < 0 \Big)
  \end{aligned}
  \end{equation*}
  Recall our assumption of function class: 

  $$
    \sup_{i \neq j} \sup_{\boldsymbol{M}, \boldsymbol{w}} \sup_{\boldsymbol{x} \in \mathcal{M}_i} |(M_i - M_j)^{T} f(\boldsymbol{x};\boldsymbol{w})| \le K.
  $$
  Then follow from the Margin Bound (Theorem.\ref{margin_bounds}), we have
  \begin{equation*}
  \begin{aligned}
    & \mathbb{P}_{\boldsymbol{x},y}\Big(\underset{c}{\arg \max} [M f(\boldsymbol{x};\boldsymbol{w})]_{c} \neq y\Big) 
    \le \sum_{i=1}^{C} p(i) \sum_{j \neq i} \mathbb{P}_{x|y=i} \Big( (M_i - M_j)^{T} f(\boldsymbol{x};\boldsymbol{w}) < 0 \Big)
    \\ \le &
    \sum_{i=1}^{C} p(i) \sum_{j \neq i} \frac{\mathfrak{R}_{N_i}(\mathcal{F})}{\gamma_{i,j}} + 
    \sum_{i=1}^{C} p(i) \sum_{j \neq i} \sqrt{\frac{\log(\log_{2} \frac{4K}{\gamma_{i,j}})}{N_c}} + 
    \\ & \ \ \ \ \ \ \ \ \ \ \ \ \ \ \ \ \ \ \ \ \ \ \ \ \ \ \ \ \ \ \ \ \
    \underbrace{
      \sum_{i=1}^{C} p(i) \sum_{j \neq i} \sqrt{\frac{\log(1 / \delta)}{2N_{i}}} 
    }_{\text{probability term}}
    + 
    \underbrace{
      \sum_{i=1}^{C} p(i) \sum_{j \neq i}
      \sum_{\boldsymbol{x} \in S_{i}} \frac{\mathbb{I}((M_i - M_j)^{T}f(\boldsymbol{x}) \le \gamma_{i,j})}{N_i}
    }_{\text{empirical risk term}}
  \end{aligned}
  \end{equation*}
  with probability at least $1-C(C-1)\delta$. 
  Then, we perform the following replace to drive the final result:
  \begin{equation*}
  \begin{aligned}
    \delta \leftarrow \frac{\delta}{C(C-1)}
  \end{aligned}
  \end{equation*}
\end{proof}

\section{Proof of Theorem.\ref{main_theorem}}\label{coveringnumberproof}

We first introduce the definition of Covering Number. 

\begin{definition}[\textbf{Covering Number} \cite{kulkarni1995rates}]
  Given $\epsilon > 0$ and $\boldsymbol{x} \in \mathbb{R}^{D}$, the open ball of radius $\epsilon$ 
  around $\boldsymbol{x}$ is denoted as 
  \begin{equation*}
  \begin{aligned}
    B_{\epsilon}(\boldsymbol{x}) = \{\boldsymbol{u} \in \mathbb{R}^{D}, \Vert \boldsymbol{u} - \boldsymbol{x} \Vert < \epsilon\}
  \end{aligned}.
  \end{equation*}
  Then the covering number $\mathcal{N}(\epsilon, A)$ of a set $A \subset \mathbb{R}^{D}$ 
  is defined as the smallest number of open balls whose union contains $A$: 
  \begin{equation*}
  \begin{aligned}
    \mathcal{N}(\epsilon, A) = \inf \left\{ k : \exists \boldsymbol{u}_1, \dots, \boldsymbol{u}_k \in \mathbb{R}^{D}, s.t. A \in 
    \bigcup_{i=1}^{k} B_{\epsilon}(\boldsymbol{u}_{i}) \right\}
  \end{aligned}
  \end{equation*}
\end{definition}

The following conclusion is demonstrated in the proof of Theorem.1 of \cite{kulkarni1995rates}. 
We use it to prove our theorem. 

\begin{theorem}[\cite{vural2017study, kulkarni1995rates}]\label{Covering_Number}
  There are $N$ samples $\{x_1, \dots, x_N\}$ drawn i.i.d from the probability measure $\mathcal{P}$. 
  Suppose the bounded support of $\mathcal{P}$ is $\mathcal{M}$, 
  then if $N$ is larger then Covering Number $\mathcal{N}(\epsilon, \mathcal{M})$, 
  we have 
  \begin{equation*}
  \begin{aligned}
    \mathbb{P}_{x}\Big(\Vert x - \hat{x}\Vert > \epsilon\Big) \le \frac{\mathcal{N}(\epsilon, \mathcal{M})}{2 N}, \forall \epsilon > 0
  \end{aligned}
  \end{equation*}
  where $\hat{x}$ is the sample that is closest to $x$ in $\{x_1, \dots, x_N\}$:
  \begin{equation*}
  \begin{aligned}
    \hat{x} \in \underset{x' \in \{x_1, \dots, x_N\}}{\arg \min} \Vert x' - x \Vert 
  \end{aligned}
  \end{equation*}
\end{theorem}

Then we provide the proof of Theorem.\ref{main_theorem}. 

\maintheoremtwo*

\begin{proof}
  We decompose the accuracy: 
  \begin{equation}\label{0}
  \begin{aligned}
    \mathbb{P}_{\boldsymbol{x},y}\Big(\max_{c} [M f(\boldsymbol{x};\boldsymbol{w})]_{c} = y\Big) = \sum_{i=1}^{C} p(i) \mathbb{P}_{\boldsymbol{x}|y=i}\Big(\max_{c} [M f(\boldsymbol{x};\boldsymbol{w})]_{c} = y\Big)
  \end{aligned}
  \end{equation}
  where $p(i)$ is the class distribution. Then, we focus on the accuracy within every class $i$.
  \begin{equation*}
  \begin{aligned}
    \mathbb{P}_{\boldsymbol{x}|y=i}(\max_{c} [M f(\boldsymbol{x};\boldsymbol{w})]_{c} = y) = \mathbb{P}_{x|y=i}( \{ (M_i - M_j)^{T} f(\boldsymbol{x};\boldsymbol{w}) > 0 \ \  \text{for any} \ \ j \neq i \} )
  \end{aligned}
  \end{equation*}
  We select the data that is closest to $\boldsymbol{x}$ in $i$ class samples $S_i$, and denote it as
  $$
  \hat{\boldsymbol{x}}(S_i) = \underset{x_1 \in S_i}{\arg \min} \Vert x_1 - x \Vert
  $$
  According to the linear separability, 
  \begin{equation*}
  \begin{aligned}
    (M_i - M_j)^{T}f(\hat{\boldsymbol{x}}(S_i);\boldsymbol{w}) \ge \gamma_{i,j}, \forall j \neq i
  \end{aligned}
  \end{equation*}
  For any $j \neq i$, we have 
  \begin{equation}\label{a}
  \begin{aligned}
    (M_i - M_j)^{T}f(\boldsymbol{x};\boldsymbol{w}) & = (M_i - M_j)^{T} (f(\boldsymbol{x};\boldsymbol{w}) + f(\hat{\boldsymbol{x}}(S_i);w) - f(\hat{\boldsymbol{x}}(S_i);w))
    \\ & = (M_i - M_j)^{T} f(\hat{\boldsymbol{x}}(S_i);w) + (M_i - M_j)^{T} (f(\boldsymbol{x};\boldsymbol{w}) - f(\hat{\boldsymbol{x}}(S_i);w))
    \\ & \ge \gamma_{i,j} - \Vert M_i - M_j \Vert \Vert f(\boldsymbol{x};\boldsymbol{w}) - f(\hat{\boldsymbol{x}}(S_i);w) \Vert
    \\ & \ge \gamma_{i,j} - L \Vert M_i - M_j \Vert \Vert \boldsymbol{x} - \hat{\boldsymbol{x}}(S_i) \Vert
  \end{aligned}
  \end{equation}
  The prediction result is related to the distance between $\boldsymbol{x}$ and $\hat{\boldsymbol{x}}(S_i)$. 
  According to Theorem.\ref{Covering_Number}, we know 
  \begin{equation*}
  \begin{aligned}
    \mathbb{P}_{\boldsymbol{x}|y=i}\Big(\Vert \boldsymbol{x} - \hat{\boldsymbol{x}}(S_i)\Vert > \epsilon\Big) \le \frac{\mathcal{N}(\epsilon, \mathcal{M}_i)}{2 N_i}
  \end{aligned}
  \end{equation*}
  To obtain the correct prediction result, $i.e.$, assure $(\ref{a}) > 0$ for all $j \neq i$, we choose $\epsilon < \min_{j \neq i} \frac{\gamma_{i,j}}{L\Vert M_i - M_j \Vert}$.
  Therefore, we have 
  \begin{equation}\label{b}
  \begin{aligned}
    \mathbb{P}_{\boldsymbol{x}|y=i} \Bigg( \left\{ (M_i - M_j)^{T}f(\boldsymbol{x};\boldsymbol{w}) > 0, \forall j \neq i \right\} \Bigg) 
    & \ge
    \mathbb{P}_{\boldsymbol{x}|y=i} \Bigg(\Vert \boldsymbol{x} - \hat{\boldsymbol{x}}(S_i)\Vert < \min_{j \neq i} \frac{\gamma_{i,j}}{L\Vert M_i - M_j \Vert} \Bigg) 
    \\ & > 1 - \frac{\mathcal{N}(\min_{j \neq i} \frac{\gamma_{i,j}}{L\Vert M_i - M_j \Vert}, \mathcal{M}_i)}{2 N_i}
  \end{aligned}
  \end{equation}
  Plug (\ref{b}) into (\ref{0}) to derive
  \begin{equation*}
  \begin{aligned}
    \mathbb{P}_{\boldsymbol{x},y}\Big(\max_{c} [M f(\boldsymbol{x};\boldsymbol{w})]_{c} = y\Big) 
    & > 1 - \sum_{i=1}^{C} p(i) \frac{\mathcal{N}(\min_{j \neq i} \frac{\gamma_{i,j}}{L\Vert M_i - M_j \Vert}, \mathcal{M}_i)}{2 N_i}
    \\ & = 1 - \sum_{i=1}^{C} p(i) \frac{\max_{j \neq i} \mathcal{N}(\frac{\gamma_{i,j}}{L\Vert M_i - M_j \Vert}, \mathcal{M}_i)}{2 N_i}
  \end{aligned}
  \end{equation*}
  Then using the conlcusions of NeurCol, if maximal norm is denoted as $\rho$, we have 
  \begin{equation*}
  \begin{aligned}
    \frac{\gamma_{i,j}}{L\Vert M_i - M_j \Vert} 
    = 
    \frac{\rho^{2} - M_i^T M_j}{L \sqrt{2\rho^{2} - 2M_i^T M_j}} 
    = 
    \frac{1}{L} \sqrt{\frac{\rho^{2} - M_i^T M_j}{2}}
  \end{aligned}
  \end{equation*}
\end{proof}

\section{Some Details of Experiments}

\subsection{Training Detail}\label{TrainDetails}

In the experiments of Section.\ref{Experiments_main}, we follow \cite{PNAS2020}'s practice. 
For all experiments, we minimize cross entropy loss using 
stochastic gradient descent with 
epoch $200$, momentum $0.9$, batch size $256$ and weight decay $5\times 10^{-4}$. 
Besides, the learning rate is set as $5\times 10^{-2}$ and annealed by ten-fold at $120$ and $160$ epoch for every dataset. 
As for data preprocess, we only perform standard pixel-wise mean subtracting and deviation dividing on images. 
To achieve $100\%$ accuracy on training set, we only use RandomFilp augmentation. 

\subsection{Generation of Equivalent \textit{Grassmannian Frame}}\label{code}

We use Code.\ref{code_} to generate Equivalent \textit{Grassmannian Frame} with 
different directions and orders. Note that the generation of rotation matrix uses
the method of \cite{lezcano2019trivializations, lezcano2019cheap}. 

\begin{codeblock}
  \begin{minted}[
    fontsize=\scriptsize
  ]{Python}
  import torch
  SEED = 1000
  torch.manual_seed(SEED)
  torch.backends.cudnn.deterministic = True
  torch.backends.cudnn.benchmark = False

  def generate_permutation_matrix(dimension):
      return torch.eye(dimension)[torch.randperm(dimension), :]

  def generate_rotation_matrix(dimension):
      A = torch.randn(dimension, dimension)
      return torch.linalg.matrix_exp(A - A.T)

  def generate_grass_frame(class_num, feature_num):
      feature = torch.nn.Linear(class_num, feature_num)
      classifier = torch.nn.Linear(feature_num, class_num)
      labels = torch.range(0, class_num-1).long()
      softmax = torch.nn.Softmax(dim=0)
      optimizer = torch.optim.SGD([
          {"params": feature.parameters(), "lr":0.1},
          {"params": classifier.parameters(), "lr":0.1}
      ])
      for i in range(1000):
          pred = torch.mm(classifier.weight , feature.weight)
          loss_ce = torch.nn.functional.cross_entropy(input=pred.T, target=labels)
          loss_l2 = 1e-1 * (torch.norm(classifier.weight) + torch.norm(feature.weight))
          loss = loss_l2 + loss_ce
          feature.zero_grad()
          classifier.zero_grad()
          loss.backward()
          optimizer.step()
          print("index: {} loss: {}".format(i, loss.item()))

      print(feature.weight.shape) # [feature_num, class_num]
      return feature.weight

  if __name__ == "__main__":
      class_num = 4
      feature_num = 2
      grass_frame = generate_grass_frame(class_num, feature_num)
      permutation = generate_permutation_matrix(class_num)
      rotation = generate_rotation_matrix(feature_num)
      grass_frame = rotation @ grass_frame @ permutation
      torch.save(grass_frame, "save_path")
\end{minted}
  \caption{Generation of Equivalent \textit{Grassmannian Frame}}\label{code_}
\end{codeblock}

\section{Numerical Simulation and Visualization of \textit{Generalized Neural Collapse}}\label{GNCfigure} 
Figure.\ref{Fig2} shows the results of a numerical simulation experiment conducted to illustrate 
the convergence of \textit{Generalized Neural Collapse}. A GIF version of Figure.\ref{Fig2} can be found \href{https://www.mediafire.com/view/1tjzxx1b8t70xss/Fig2.gif/file}{HERE}.
During the simulation, we discovered that the condition $\rho \rightarrow \infty$, which is
believed to be necessary for the occurrence of \textit{Grassmannian Frame} (in the \textbf{NC3} of Theorem.\ref{CE_grass}), 
may not be required. 
This suggests that there may be other ways to prove \textit{Generalized Neural Collapse} with fewer assumptions.

\begin{figure*}[htbp]
  
  \begin{minipage}[t]{1.0\linewidth}
  \centering
    \begin{minipage}[t]{0.3\linewidth}
      \centering  
      \includegraphics[width=4cm]{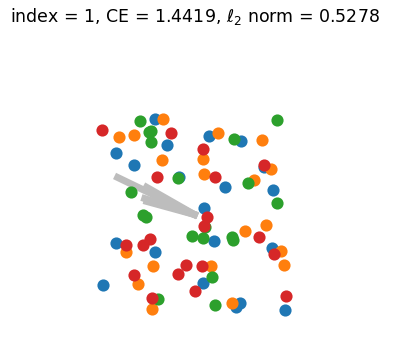}
    \end{minipage}
    \begin{minipage}[t]{0.3\linewidth}
      \centering
      \includegraphics[width=4cm]{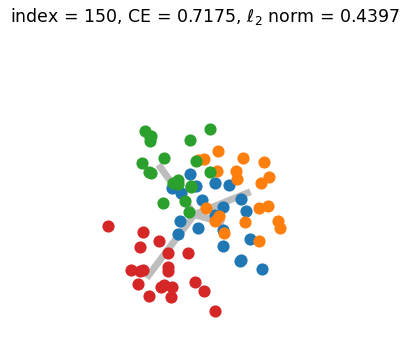}
    \end{minipage}
    \begin{minipage}[t]{0.3\linewidth}
      \centering
      \includegraphics[width=4cm]{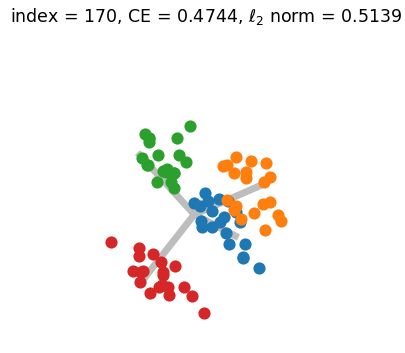}
    \end{minipage}
    \\
    \begin{minipage}[t]{0.3\linewidth}
      \centering
      \includegraphics[width=4cm]{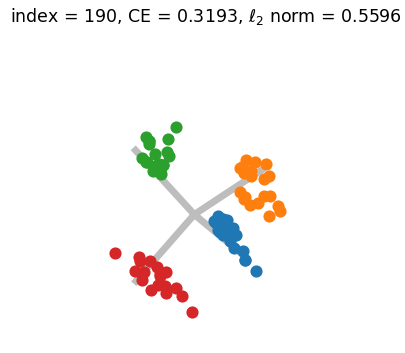}
    \end{minipage}
    \begin{minipage}[t]{0.3\linewidth} 
      \centering
      \includegraphics[width=4cm]{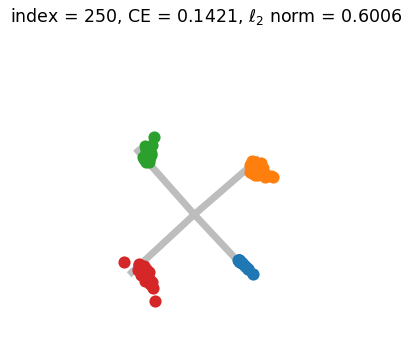}
    \end{minipage}
    \begin{minipage}[t]{0.3\linewidth}
      \centering
      \includegraphics[width=4cm]{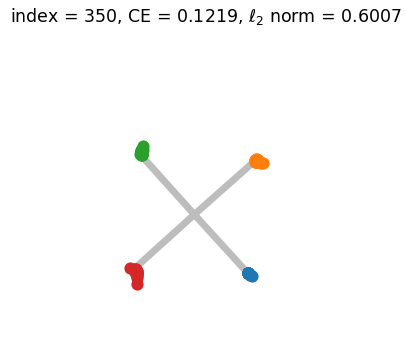}
    \end{minipage}
  \end{minipage}
   \caption{\footnotesize 
      Visualization of \textit{Generalized Neural Collapse}. 
      There are $4$ classes, and the feature space is $2$-dimensional. 
      Every class has $20$ samples. 
      In the figures, the points with different colors represent the features of samples from different classes, 
      and the lines indicate the linear classifier. 
   }\label{Fig2}
 \end{figure*}

\end{document}